\documentclass{article} 

\usepackage{amsmath, times, amssymb, amsthm}
\usepackage{verbatim, enumerate, color}
\usepackage{hyperref, url, graphicx}
\usepackage{mathtools, wrapfig}
\usepackage[caption=false]{subfig}
\usepackage{epsfig,multirow}
\usepackage{xcolor, comment, natbib}
\usepackage[margin=1in]{geometry}

\newtheorem*{assume*}{\assumenumber}
\providecommand{\assumenumber}{}
\makeatletter
\newenvironment{assume}[2]
 {%
  \renewcommand{\assumenumber}{Assumption (#1${\bf #2}$)}%
  \begin{assume*}%
  \protected@edef\@currentlabel{#1${#2}$}%
 }
 {%
  \end{assume*}
 }
\makeatother

\DeclareMathOperator{\Var}{Var}
	
\theoremstyle{plain}
\newtheorem{theorem}{Theorem}[section]
\newtheorem{corollary}[theorem]{Corollary}
\newtheorem{lemma}[theorem]{Lemma}

\theoremstyle{definition}
\newtheorem*{definition}{Definition}

\DeclarePairedDelimiter{\ceil}{\lceil}{\rceil}

\begin{document}

\title{Convergence of gradient based pre-training \\ in Denoising autoencoders}
\date{}
\author{Vamsi K. Ithapu$^1$, Sathya Ravi$^2$, Vikas Singh$^{3,1}$\\
       $^1$Computer Sciences \hspace{2mm} 
       $^2$Industrial and Systems \hspace{2mm}
       $^3$Biostat. and Med. Informatics\\ 
       University of Wisconsin-Madison}
\maketitle

\begin{abstract}
The success of deep architectures is at least in part attributed to the layer-by-layer unsupervised pre-training that initializes the network.
Various papers have reported extensive empirical analysis focusing on the design and implementation of good pre-training procedures.
However, an understanding pertaining to the consistency of parameter estimates, the convergence of learning procedures and the sample size estimates is still unavailable in the literature.
In this work, we study pre-training in classical and distributed denoising autoencoders with these goals in mind. 
We show that the gradient converges at the rate of $\frac{1}{\sqrt{N}}$ and has a sub-linear dependence on the size of the autoencoder network.
In a distributed setting where disjoint sections of the whole network are pre-trained synchronously, we show that the convergence improves by at least $\tau^{3/4}$, where $\tau$ corresponds to the size of the sections. 
We provide a broad set of experiments to empirically evaluate the suggested behavior.
\end{abstract}

\section{Introduction}
\label{sec:intro}

In the last decade, deep learning models have provided state of the art results for a broad spectrum of problems in computer vision \cite{krizhevsky2012imagenet, taigman2014deepface}, natural language processing \cite{socher2011dynamic, socher2011semi}, machine learning \cite{hamel2010learning, dahl2011large} and biomedical imaging \cite{plis2013deep}. 
The underlying {\em deep} architecture with multiple layers of hidden variables allows for learning high-level representations which fall beyond the hypotheses space of ({\em shallow}) alternatives \cite{bengio2009learning}.
This representation-learning behavior is attractive in many applications where setting up a suitable feature engineering pipeline that captures the discriminative content of the data remains difficult, but is critical to the overall performance.
Despite many desirable qualities, the richness afforded by multiple levels of variables and the non-convexity of the learning objectives makes training deep architectures challenging. 
An interesting solution to this problem proposed in \cite{hinton2006reducing, bengio2007greedy} is a hybrid two-stage procedure. 
The first step performs a layer-wise unsupervised learning, referred to as ``pre-training'', which provides a suitable initialization of the parameters. 
With this warm start, the subsequent discriminative (supervised) step simply {\it fine-tunes} the network with an appropriate loss function.
Such procedures broadly fall under two categories -- restricted Boltzmann machines and autoencoders \cite{bengio2009learning}. 
Extensive empirical evidence has demonstrated the benefits of this strategy, and the recent success of deep learning is at least partly attributed to pre-training \cite{bengio2009learning, erhan2010does, coates2011analysis}.

Given this role of pre-training, there is significant interest in understanding precisely what the unsupervised phase does and why it works well.
Several authors have provided interesting explanations to these questions.  
\cite{bengio2009learning} interprets pre-training as providing the downstream optimization with a suitable initialization. 
\cite{erhan2009difficulty, erhan2010does} presented compelling empirical evidence that pre-training serves as an ``unusual form of regularization'' which biases the parameter search by minimizing variance. 
The influence of the network structure (lengths of visible and hidden layers) and optimization methods on the pre-training estimates have been well studied \cite{coates2011analysis, ngiam2011optimization}.
\cite{dahl2011large} evaluate the role of pre-training for DBN-HMMs as a function of sample sizes and discuss the regimes which yield the maximum improvements in performance.
A related but distinct set of results describe procedures that construct ``meaningful'' data representations. 
Denoising autoencoders \cite{vincent2010stacked} seek representations that are invariant to data corruption, while contractive autoencoders (CA) \cite{rifai2011contractive} seek robustness to data variations. 
The manifold tangent classifier \cite{rifai2011manifold} searches for low dimensional non-linear sub-manifold that approximates the input distribution. 
Other works have shown that with a suitable architecture, even a random initialization seems to give impressive performance \cite{saxe2011random}. 
Very recently, \cite{livni2014computational, bianchini2014nnls} have analyzed the complexity of multi-layer neural networks, theoretically justifying that certain types of deep networks learn complex concepts. 
While the significance of the results above cannot be overemphasized, our current understanding of the conditions under which pre-training is {\it guaranteed} to work well is still not very mature. 
Our goal here is to complement the above body of work by deriving specific conditions under which this pre-training procedure will have convergence guarantees. 

To keep the presentation simple, we restrict our attention to a widely used form of pre-training --- Denoising autoencoder --- as a sandbox to develop our main ideas, while noting that a similar style of analysis is possible for other (unsupervised) formulations also. 
Denoising auto-encoders (DA) seek robustness to partial destruction (or corruption) of the inputs, implying that a {\it good} higher level representation must characterize only the `stable' dependencies among the data dimensions (features) and remain invariant to small variations \cite{vincent2010stacked}. 
Since the downstream layers correspond to increasingly non-linear compositions, the layer-wise unsupervised pre-training with DAs gives increasingly abstract representations of the data as the depth (number of layers) increases. 
These non-linear transformations (e.g., sigmoid functions) make the objective non-convex, and so DAs are typically optimized via a stochastic gradients. 
Recently, large scale architectures have also been successfully trained in a massively distributed setting where the stochastic descent is performed asynchronously 
over a cluster \cite{dean2012large}.  
The empirical evidence regarding the performance of this scheme is compelling. 
The analysis in this paper is an attempt to understand this behavior on the theoretical side (for both classical and distributed DA), and identify situations where such constructions {\em will} work well with certain guarantees. 

We summarize the main {\bf contributions} of this paper. 
We first derive {\em convergence results and the associated sample size estimates} of pre-training a single layer DA using the randomized stochastic gradients \cite{ghadimi2013stochastic}. 
We show that the convergence of expected gradients is $\mathcal{O}\left(\frac{(d_hd_v)^{3/4}}{\sqrt{N}}\right)$
and the number of calls (to a first order oracle) is $\mathcal{O}\left(\frac{(d_hd_v)^{3/2}}{\epsilon^2}\right)$, where $d_h$ and $d_v$ correspond to the number of hidden and visible layers, $N$ is the number of iterations, and $\epsilon$ is an error parameter. 
We then show that the DA objective can be distributed and present improved rates while learning small fractions of the network synchronously. 
These bounds provide a nice relationship between the sample size, asymptotic convergence of gradient norm (to zero) and the number of hidden/visible units. 
Our results extend easily to stacked and convolutional denoising auto-encoders. 
Finally, we provide sets of experiments to evaluate if the results are meaningful in practice. 

\section{Preliminaries}
\label{sec:prelim}

Autoencoders are single layer neural networks that learn over-complete representations by applying nonlinear transformations on the input data \cite{vincent2010stacked, bengio2009learning}.
Given an input ${\bf x}$, an autoencoder identifies representations of the form  ${\bf h} = \sigma({\bf W}{\bf x})$, where ${\bf W}$ is a $d_h \times d_v$ transformation matrix and $\sigma$ denotes point--wise sigmoid nonlinearity. 
Here, $d_v$ and $d_h$ denote the lengths of visible and hidden layers respectively.
Various types of autoencoders are possible depending on the assumptions that generate the ${\bf h}$'s --- robustness to data variations/corruptions, enforcing data to lie on some low-dimensional sub-manifolds etc. \cite{rifai2011contractive, rifai2011manifold}. 

Denoising autoencoders are widely used class of autoencoders \cite{vincent2010stacked}, that learn higher-level representations by leveraging the inherent correlations/dependencies among input dimensions ($j = 1,\ldots,d_v$), thereby ensuring that ${\bf h}$ is robust to changes in less informative input/visible units. 
This is based on the hypothesis that abstract high-level representations should only encode stable data dependencies across input dimensions, and be robust to spurious correlations and invariant features.  
This is done by `corrupting' each individual visible dimension randomly, and using the corrupted version ($\tilde{{\bf x}}$'s) instead, to learn ${\bf h}$'s. 
The corruption generally corresponds to ignoring (setting to $0$) the input signal with some probability (denoted by $\zeta$), although other types of additive/multiplicative corruption may also be used.  
If ${\bf x}_j$ is the input at the $j^{th}$ unit, then the corrupted signal is $\tilde{{\bf x}}_j = {\bf x}_j$ with probability $1-\zeta$ and $0$ otherwise where $j = 1,\ldots,d_v$.
Note that each of the $d_v$ dimensions are corrupted independently with the same probability $\zeta$.
DA pre-training then corresponds to estimating the transformation ${\bf W}$ by minimizing the following objective \cite{bengio2009learning},
\begin{equation}\label{eq:damin}
  \min_{{\bf W}} \hspace{3mm} \mathbb{E}_{p({\bf x},{\bf \tilde{x}})} \| {\bf x} - \sigma({\bf W}^T \sigma({\bf W\tilde{x}})) \|^{2}
\end{equation}
where the expectation is over the joint probability $p({\bf x},\tilde{\bf x})$ of sampling an input ${\bf x} \sim \mathcal{D}$ and generating the corresponding $\tilde{\bf x} | {\bf x}$ using $\zeta$. 
The bias term (which is never corrupted) is taken care of by appending inputs ${\bf x}$ with $1$. 

For notational simplicity, let us denote the process of generating $\{{\bf x},\tilde{{\bf x}}\}$ by a random variable $\eta$, i.e., one sample of $\eta$ corresponds to a pair $\{{\bf x},\tilde{{\bf x}}\}$ where $\tilde{{\bf x}}$ is constructed by randomly corrupting each of the $d_v$ dimensions of ${\bf x}$ with some probability $\zeta$.
Then, if the reconstruction loss is $\mathcal{L}(\eta; {\bf W}) := \mathcal{L}(\tilde{{\bf x}},{\bf x}; {\bf W}) = \| {\bf x} - \sigma({\bf W}^T \sigma({\bf W}\tilde{{\bf x}})) \|^2$, the objective in (\ref{eq:damin}) becomes
\begin{equation}\label{eq:daminrsg}
\min_{{\bf W}} \hspace{3mm} f({\bf W}) := \mathbb{E}_{\eta} \mathcal{L}(\eta;{\bf W})
\end{equation} 
Observe that the loss $\mathcal{L}(\eta; {\bf W})$ and the objective in (\ref{eq:daminrsg}) constitutes an expectation over the randomly corrupted sample pairs $\eta := \{{\bf x},\tilde{\bf x} \}$, which is non-convex. 
Analyzing convergence properties of such an objective using classical techniques, especially in a (distributed) stochastic gradient setup, is difficult. 
Therefore, given that the loss function is a composition of sigmoids, one possibility is to adopt 
convex variational relaxations of sigmoids in (\ref{eq:daminrsg}) and then apply standard convex analysis. 
But non-convexity is, in fact, the most interesting aspect of deep architectures, 
and so the analysis of a loose convex relaxation will be unable to explain the empirical success of DAs, and deep learning in general.

{\em High Level Idea.}
The starting point of our analysis is a very recent result on stochastic gradients which only makes a weaker assumption of Lipschitz differentiability of the objective (rather than convexity). 
We assume that the optimization of (\ref{eq:daminrsg}) proceeds by querying a stochastic first order oracle ($\mathcal{SFO}$), which provides noisy gradients of the objective function. 
For instance, the $\mathcal{SFO}$ may simply compute a noisy gradient with a single sample $\eta^k := \{{\bf x_s},\tilde{{\bf x}}_s\}$ at the $k^{th}$ iteration and use that alone to evaluate $\nabla_{{\bf W}} \mathcal{L}(\eta^k;{\bf W}^k)$.
The main idea adapted from \cite{ghadimi2013stochastic} to our problem is to express the stopping criterion for the gradient updates by a probability distribution $\mathbb{P}_R(\cdot)$ over iterations $k$, i.e., the stopping iteration is $k \sim \mathbb{P}_R(\cdot)$ (and hence the name randomized stochastic gradients, RSG). 
Observe that this is the only difference from classical stochastic gradients used in pre-training, where the stopping criterion is assumed to be the last iteration.
RSG will offer more useful theoretical properties, and is a negligible practical change to existing implementations. 
This then allows us to compute the expectation of the gradient norm, where the expectation is over stopping iterations sampled according to $\mathbb{P}_R(\cdot)$. 
For our case, the updates are given by,
\begin{equation}\label{eq:rsgupdate}
{\bf W}^{k+1} \leftarrow {\bf W}^{k} - \gamma^{k} G(\eta^k;{\bf W}^k)
\end{equation}
where, $G(\eta^k;{\bf W}^k) = \nabla_{{\bf W}} \mathcal{L}(\eta^k;{\bf W}^k)$ is the noisy gradient computed at $k^{th}$ iteration ($\gamma^k$ is the stepsize).
We have flexibility in specifying the distribution of stopping criterion $\mathbb{P}_R(\cdot)$. 
It can be fixed a priori or selected by a hyper-training procedure that chooses the best $\mathbb{P}_R(\cdot)$ (based on an accuracy measure) from a pool of distributions $\mathcal{P}_R$. 
With these basic tools in hand, we first compute the expectation of gradients where the expectation accounts for both the stopping criterion $k \sim \mathbb{P}_R(\cdot)$ and $\eta := \{{\bf x},\tilde{\bf x} \}$.
We show that if the stepsizes $\gamma^k$ in (\ref{eq:rsgupdate}) are chosen carefully, the expected gradients decrease monotonically and converge. 
Based on this analysis, we derive the rate of convergence and corresponding sample size estimates for DA pre-training. 
We describe the one--layer DA (i.e., with one hidden layer) in detail, and all our results extend easily to the stacked and convolutional settings since the pre-training is done layer-wise in multi-layer architectures.

\section{Denoising Autoencoders (DA) pre-training}
\label{sec:uppersamp}

We first present some results on the continuity and boundedness of the objective $f({\bf W})$ in (\ref{eq:daminrsg}), followed by the convergence rates for the optimization. 
Denote the element in $i^{th}$ row and $j^{th}$ column of ${\bf W}$ by ${\bf W}_{ij}$ where $i=1,\ldots,d_{v}$ and $j=1,\ldots,d_h$. 
We require the following Lipschitz continuity assumptions on $\mathcal{L}(\eta;{\bf W}_{ij})$ and the gradient $\nabla_{{\bf W}_{ij}}f({\bf W}_{ij})$, which are fairly common in numerical optimization. $L$ and $L'$ are Lipschitz constants. 
 \begin{assume}{A}{1} \label{as:1}
   \begin{equation*}\label{as:1}
    \| \mathcal{L}(\eta;{\bf W}_{ij}) - \mathcal{L}(\eta;\hat{{\bf W}}_{ij}) \| \leq L \| {\bf W}_{ij} - \hat{{\bf W}}_{ij}\| \quad \forall  i,j 
  \end{equation*}
\end{assume}
\begin{assume}{A}{2} \label{as:2}
   \begin{equation*}\label{as:2}
    \| \nabla_{{\bf W}_{ij}}f({\bf W}_{ij}) - \nabla_{{\bf W}_{ij}}f(\hat{{\bf W}}_{ij}) \| \leq L' \| {\bf W}_{ij} - \hat{{\bf W}}_{ij}\| \quad \forall  i,j, 
  \end{equation*}
\end{assume}
We see from (\ref{eq:damin}) that $\mathcal{L}(\eta;{\bf W}_{ij})$ is symmetric in ${\bf W}_{ij}, \forall i,j$.
Depending on where ${\bf W}_{ij}$ is located in the parameter space (and the variance of each data dimension $j$), each $\mathcal{L}(\eta;{\bf W}_{ij})$ corresponds to some $L_{ij}$, and $L$ will then be the maximum of all such $L_{ij}$'s (similarly for $L'$). 

Based on the definition of $G(\eta^k;{\bf W}^k)$ and (\ref{eq:daminrsg}), we see that the noisy gradients $G(\eta^k;{\bf W}^k)$ are unbiased estimates of the true gradient since $\nabla_{{\bf W}}f({\bf W}^k) = \mathbb{E}_{\eta^k} G(\eta^k;{\bf W}^k)$. 
To compute the expectation of the gradients, $\nabla_{{\bf W}}f({\bf W}^k)$, over the distribution governing whether the process stops at iteration $k$, i.e., $R \sim \mathbb{P}_R(\cdot)$, we first 
state a result regarding the variance of the noisy gradients and the Lipschitz constant of $\nabla_{{\bf W}}f({\bf W}^k)$. 
All proofs are included in the supplement. 

\vspace{2mm}
\begin{lemma}[\bf Variance bound and Lipschitz constant] \label{thm:varlip}
  Using \ref{as:1}, \ref{as:2} and $\nabla_{{\bf W}}f({\bf W}^k) = \mathbb{E}_{\eta^k} G(\eta^k;{\bf W}^k)$, we have
  %
  \begin{equation}\label{eq:varlipbound} 
  \begin{aligned}
  \Var(G(\eta^k;{\bf W}^k)) &\leq d_hd_vL^2 \\ 
  \| \nabla_{{\bf W}} f({\bf W}) - \nabla_{{\bf W}} f(\hat{{\bf W}}) \| &\leq \sqrt{d_hd_v}L' \|{\bf W} - \hat{{\bf W}}\|	
  \end{aligned} 
  \end{equation}
\end{lemma}
\begin{proof}
Recall that the assumptions ${\bf [A1]}$ and ${\bf [A2]}$ are,
\begin{equation*} \begin{split}
& {\bf [A1]} \hspace{5mm} \| \mathcal{L}(\eta;{\bf W}_{ij}) - \mathcal{L}(\eta;\hat{{\bf W}}_{ij}) \| \leq L \| {\bf W}_{ij} - \hat{{\bf W}}_{ij}\| \hspace{2mm} \forall \hspace{2mm} i,j \\
& {\bf [A2]} \hspace{5mm} \| \nabla_{{\bf W}_{ij}}f({\bf W}_{ij}) - \nabla_{{\bf W}_{ij}}f(\hat{{\bf W}}_{ij}) \| \leq L' \| {\bf W}_{ij} - \hat{{\bf W}}_{ij}\| \hspace{2mm} \forall \hspace{2mm} i,j 
\end{split} \end{equation*} 
The noisy gradient is defined as $G(\eta^k;{\bf W}^k) = \nabla_{{\bf W}} \mathcal{L}(\eta^k;{\bf W}^k)$.
Using the mean value theorem and {\bf [A1]}, we have $|G(\eta^k;{\bf W}_{ij}^k)| \leq L$.
This implies that the maximum variance of $G(\eta^k;{\bf W}_{ij}^k)$ is $L^2$.
We can then obtain the following upper bound on the variance of $G(\eta^k;{\bf W}^k)$,
\begin{equation}\label{eq:varbound} \begin{split}
\mathbb{E}_{\eta^k} (\| G(\eta^k;{\bf W}^k) - \nabla_{{\bf W}} f({\bf W}^k)\|^2) & = \mathbb{E}_{\eta^k} (\sum_{ij} (G(\eta^k;{\bf W}_{ij}^k) - \nabla_{{\bf W}_{ij}} f({\bf W}_{ij}^k))^2) \\
& = \sum_{ij} Var(G(\eta^k;{\bf W}_{ij}^k)) \leq d_hd_vL^2
\end{split} \end{equation}
Using {\bf [A2]}, we have 
\begin{equation}\label{eq:lipbound} \begin{split}
\| \nabla_{{\bf W}} f({\bf W}) - \nabla_{{\bf W}} f(\hat{{\bf W}}) \|^2 & = \sum_{i,j} \| \nabla_{{\bf W}_{ij}} f({\bf W}_{ij}) - \nabla_{{\bf W}_{ij}} f(\hat{{\bf W}}_{ij}) \|^2 \\
& \leq \sum_{i,j} (L'_{ij})^2 \| {\bf W}_{ij} - \hat{{\bf W}}_{ij} \|^2 \leq \sum_{i,j} (L'_{ij})^2 \sum_{i',j'}\| {\bf W}_{i'j'} - \hat{{\bf W}}_{i'j'} \|^2 \\
& \leq d_hd_vL'^2 \|{\bf W} - \hat{{\bf W}}\|^2	
\end{split} \end{equation}
where the equality follows from the definition of $\ell_2$-norm. The second inequality is from ${\bf [A2]}$.
The last two inequalities use the definition of $\ell_2$-norm and that $L'$ is the maximum of all $L'_{ij}$s.
\end{proof}

Whenever the inputs ${\bf x}$ are bounded between $0$ and $1$, $f({\bf W})$ is finite-valued everywhere and there exists a minimum due to the bounded range of sigmoid in (\ref{eq:damin}).
Also, $f(\cdot)$ is analytic with respect to ${\bf W}_{ij} \forall i,j$. 
Now, if one adopts the RSG scheme for the optimization, using Lemma \ref{thm:varlip}, we have the following upper bound on the expected gradients for the one--layer DA pre-training in (\ref{eq:daminrsg}). 

\vspace{2mm}
\begin{lemma}[\bf Expected gradients of one--layer DA] \label{thm:expgrad}
  Let $N \geq 1$ be the maximum number of RSG iterations with step sizes $\gamma^k < \frac{2}{L'\sqrt{d_hd_v}}$. Let $\mathbb{P}_R(\cdot)$ be given as
  \begin{equation} \label{eq:pmf} \begin{aligned}
    \mathbb{P}_R(k) := &Pr(R = k) \\ 
    &= \frac{2\gamma^k - L'\sqrt{d_hd_v}(\gamma^k)^2}{\sum_{k=1}^{N} \left(2\gamma^k - L'\sqrt{d_hd_v}(\gamma^k)^2\right)} 
  \end{aligned} \end{equation}
  where $k = 1,\ldots,N$. If $D_f = 2(f({\bf W}^1) - f^*)$, we have
  \begin{equation} \label{eq:upperbound} \begin{aligned}
  \mathbb{E} &\left(\| \nabla_{\bf W} f({\bf W}^R)\|^2\right) \hspace{4mm} \\ 
  &\leq 
  \frac{D_f + \left(\sqrt{d_hd_v}\right)^{3}L^2L' \sum_{k=1}^N(\gamma^k)^2 }{\sum_{k=1}^N \left(2\gamma^k - L'\sqrt{d_hd_v}(\gamma^k)^2\right)}
  \end{aligned} 
  \end{equation}
\end{lemma}
\begin{proof} 
  Broadly, this proof emulates the proof of Theorem 2.1 in \cite{ghadimi2013stochastic} with several adjustments. 
  The Lipschitz continuity assumptions (refer to ${\bf [A1]}$ and ${\bf [A2]}$) give the following bounds on the variance of $G(\eta^k;{\bf W}^k)$ and the Lipschitz continuity of $\nabla_{{\bf W}} f({\bf W})$ 
  (refer to Lemma \ref{thm:varlip}),
  \begin{equation} \begin{aligned}\label{eq:varlipbounds}
      \Var(G(\eta^k;{\bf W}^k)) &\leq d_hd_vL^2 \\
      \| \nabla_{{\bf W}} f({\bf W}) - \nabla_{{\bf W}} f(\hat{{\bf W}}) \| &\leq \sqrt{d_hd_v}L' \|{\bf W} - \hat{{\bf W}}\| \\
    \end{aligned} \end{equation}
  Using the properties of Lipschitz continuity we have,
  \begin{equation*} 
    f({\bf W}^{k+1}) \leq f({\bf W}^{k}) + \langle \nabla_{\bf W} f({\bf W}^{k}),{\bf W}^{k+1}-{\bf W}^{k} \rangle + \frac{\sqrt{d_hd_v}L'}{2} \|{\bf W}^{k+1}-{\bf W}^{k}\|^2     
\end{equation*}
Since the update of ${\bf W}^k$ using the noisy gradient is ${\bf W}^{k+1} \leftarrow {\bf W}^k - \gamma^k G(\eta^k;{\bf W}^k)$, where $\gamma^k$ is the step--size, we then have,
\begin{equation*} 
f({\bf W}^{k+1}) \leq f({\bf W}^{k}) - \gamma^k \langle \nabla_{\bf W} f({\bf W}^{k}),G(\eta^k;{\bf W}^k) \rangle + \frac{\sqrt{d_hd_v}L'}{2}(\gamma^k)^2 \|G(\eta^k;{\bf W}^k)\|^2 
\end{equation*}
By denoting $\delta^k := G(\eta^k;{\bf W}^k) - \nabla_{\bf W} f({\bf W}^{k})$,
\begin{equation*} \begin{aligned}
f({\bf W}^{k+1}) &\leq f({\bf W}^{k}) - \gamma^k \|\nabla_{\bf W} f({\bf W}^{k})\|^2 - \gamma^k \langle \nabla_{\bf W} f({\bf W}^{k}),\delta^k \rangle \\
& \hspace{5mm} + \frac{\sqrt{d_hd_v}L'}{2}(\gamma^k)^2 \biggl( \|\nabla_{\bf W} f({\bf W}^{k})\|^2 + 2\langle \nabla_{\bf W} f({\bf W}^{k}),\delta^k \rangle + \|\delta^k\|^2 \biggr) \\
\end{aligned} \end{equation*}
Rearranging terms on the right hand side above,
\begin{equation*} \begin{aligned}
f({\bf W}^{k+1}) &\leq f({\bf W}^{k}) - \left( \gamma^k - \frac{\sqrt{d_hd_v}L'}{2}(\gamma^k)^2 \right) \|\nabla_{\bf W} f({\bf W}^{k})\|^2 \\
& \hspace{5mm} - \biggl( \gamma^k - \sqrt{d_hd_v}L'(\gamma^k)^2 \biggr) \langle \nabla_{\bf W} f({\bf W}^{k}),\delta^k \rangle + \frac{\sqrt{d_hd_v}L'}{2} (\gamma^k)^2 \|\delta^k\|^2 \\
\end{aligned} \end{equation*}
Summing the above inequality for $k = 1,\ldots,N$, 
\begin{equation} \begin{aligned}\label{eq:sumineq0}
\sum_{k=1}^{N} \left( \gamma^k - \frac{\sqrt{d_hd_v}L'}{2}(\gamma^k)^2 \right) \|f({\bf W}^{k})\|^2 &\leq f({\bf W}^1) - f({\bf W}^{N+1}) \\
- \sum_{k=1}^{N} \biggl( \gamma^k - \sqrt{d_hd_v}L'(\gamma^k)^2 \biggr) &\langle \nabla_{\bf W} f({\bf W}^{k}),\delta^k \rangle 
+ \frac{\sqrt{d_hd_v}L'}{2} \sum_{k=1}^{N} (\gamma^k)^2 \|\delta^k\|^2 \\
\end{aligned} \end{equation}
where ${\bf W}^0$ is the initial estimate. Using $f^{*} \leq f({\bf W}^{N+1})$, we have,
\begin{equation}\label{eq:sumineq1} \begin{aligned}
\sum_{k=1}^{N} \left( \gamma^k - \frac{\sqrt{d_hd_v}L'}{2}(\gamma^k)^2 \right) \|f({\bf W}^{k})\|^2 &\leq f({\bf W}^1) - f^{*} \\
- \sum_{k=1}^{N} \biggl( \gamma^k - \sqrt{d_hd_v}L'(\gamma^k)^2 \biggr) &\langle \nabla_{\bf W} f({\bf W}^{k}),\delta^k \rangle 
+ \frac{\sqrt{d_hd_v}L'}{2} \sum_{k=1}^{N} (\gamma^k)^2 \|\delta^k\|^2 \\
\end{aligned} \end{equation}
We now take the expectation of the above inequality over all the random variables in the RSG updating process 
-- which include the randomization $\eta$ used for constructing noisy gradients, and the stopping iteration $R \sim \mathbb{P}_R(\cdot)$.
First, note that the stopping criterion is chosen at random with some given probability $\mathbb{P}_R(\cdot)$ and is independent of $\eta$.
Second, recall that the random process $\eta$ is such that the random variable $\eta^{k}$ is independent of $\eta^{k+1}$ for some iteration number $k$, because $\mathcal{SFO}$ selects then randomly.
However, the update point ${\bf W}^{k+1}$ depends on $G(\eta^k;{\bf W}^k)$ (which are functions of the random variables $\eta^k$) from the first to the $k^{th}$ iteration. 
That is, ${\bf W}^{k+1}$ is not independent of ${\bf W}^{k}$, and in fact the updates ${\bf W}^k$ form a Markov process. 
So, we can take the expectation with respect to the joint probability $p(\eta^{[N]},R) = p(\eta^{[N]})p(R)$ where $\eta^{[N]}$ denotes the random process from $\eta^{1}$ until $\eta^N$.
We analyze each of the last two terms on the right hand side of \eqref{eq:sumineq1} by first taking expectation with respect to $\eta^{[N]}$.
The second last term becomes,
\begin{equation}\label{eq:term1} \begin{aligned}
\mathbb{E}_{\eta^{[N]}} & \left[ \sum_{k=1}^{N} \biggl( \gamma^k - \sqrt{d_hd_v}L'(\gamma^k)^2 \biggr) \langle \nabla_{\bf W} f({\bf W}^{k}),\delta^k \rangle \right] \\
& \hspace{5mm} = \sum_{k=1}^{N} \biggl( \gamma^k - \sqrt{d_hd_v}L'(\gamma^k)^2 \biggr) \hspace{2mm} \mathbb{E}_{\eta^{[k]}} (\langle \nabla_{\bf W} f({\bf W}^{k}),\delta^k \rangle | \eta^{1},\ldots,\eta^{k}) = 0 \\
\end{aligned} \end{equation}
where the last equality follows from the definition of $\delta^k = G(\eta^k;{\bf W}^k) - \nabla_{\bf W} f({\bf W}^{k})$ and $\mathbb{E}_{\eta^{k}} (G(\eta^k;{\bf W}^k)) = \nabla_{\bf W} f({\bf W}^{k}))$.
Further, from Equation \ref{eq:varlipbounds} we have $\mathbb{E}_{\eta^{k}} \| \delta^k \|^2 = Var(G(\eta^k;{\bf W}^k)) \leq d_hd_vL^2$. So, the expectation of the last term in \eqref{eq:sumineq1} becomes,
\begin{equation}\label{eq:term2} 
\mathbb{E}_{\eta^{[N]}} \left[ \frac{\sqrt{d_hd_v}L'}{2} \sum_{k=1}^{N} (\gamma^k)^2 \|\delta^k\|^2 \right] = \frac{\sqrt{d_hd_v}L'}{2} \sum_{k=1}^{N} (\gamma^k)^2 \mathbb{E}_{\eta^{[N]}} (\|\delta^k\|^2) \leq \frac{(\sqrt{d_hd_v})^3L'L^2}{2} \sum_{k=1}^{N} (\gamma^k)^2 
\end{equation}
Using \eqref{eq:term1} and \eqref{eq:term2} and the inequality in \eqref{eq:sumineq1} we have,
\begin{equation}\label{eq:sumineq2} 
\sum_{k=1}^{N} \left( 2\gamma^k - \sqrt{d_hd_v}L'(\gamma^k)^2 \right) \mathbb{E}_{\eta^{[N]}} \|f({\bf W}^{k})\|^2 
\leq 2(f({\bf W}^1) - f^{*}) + (\sqrt{d_hd_v})^3L'L^2 \sum_{k=1}^{N} (\gamma^k)^2 
\end{equation}
Using the definition of $\mathbb{P}_R(k)$ from Equation \ref{eq:pmf} and denoting $D_f = 2(f({\bf W}^1) - f^*)$, we finally obtain
\begin{equation}\label{eq:finalineq} \begin{aligned}
\mathbb{E}_{R,\eta^{[N]}} (\| \nabla_{\bf W} f({\bf W}^R)\|^2) &= \sum_{k=1}^{N} \frac{(2\gamma^k - L'\sqrt{d_hd_v}(\gamma^k)^2) \mathbb{E}_{\eta^{[N]}} (\| \nabla_{\bf W} f({\bf W}^R)\|^2)}{\sum_{k=1}^{N} (2\gamma^k - L'\sqrt{d_hd_v}(\gamma^k)^2)} \\
& \leq \frac{D_f + (\sqrt{d_hd_v})^{3}L^2L' \sum_{k=1}^N(\gamma^k)^2 }{\sum_{k=1}^N (2\gamma^k - L'\sqrt{d_hd_v}(\gamma^k)^2)}
\end{aligned} \end{equation}
\end{proof}

The expectation in (\ref{eq:upperbound}) is over $\eta$ and $R \sim \mathbb{P}_R(\cdot)$.
Here, $\gamma^k < \frac{2}{L'\sqrt{d_hd_v}}$ ensures that the summations in the denominators of $\mathbb{P}_R(\cdot)$ in (\ref{eq:pmf}) and the bound in (\ref{eq:upperbound}) are positive.
$D_f$ represents a quantity which is twice the deviation of the objective $f({\bf W})$ at the RSG starting point (${\bf W}^1$) from the optimum. 
Observe that the bound in (\ref{eq:upperbound}) is a function of $D_f$ and network parameters, and we will analyze it shortly.

As stated, there are a few caveats that are useful to point out. Since no convexity assumptions are imposed on the loss function,  
Lemma \ref{thm:expgrad} on its own offers no guarantee that the function values decrease as $N$ increases.
In particular, in the worst case, the bound may be loose. 
For instance, when $D_f \approx 0$ (i.e., the initial point is already a good estimate of the stationary point), the upper bound in (\ref{eq:upperbound}) is non--zero. 
Further, the bound contains summations involving the stepsizes, both in the numerator and denominator, indicating that the limiting behavior may be sensitive to the choices of $\gamma^k$.
The following result gives a remedy --- by choosing $\gamma^k$ to be small enough, the upper bound in (\ref{eq:upperbound}) will decrease monotonically as $N$ increases.

\vspace{2mm}
\begin{lemma}[\bf Monotonicity and convergence of expected gradients] \label{thm:convexpgrad}
  By choosing $\gamma^k$ such that
  \begin{equation}\label{eq:gammabounds}
    \gamma^{k+1} \leq \gamma^k \hspace{2mm} \text{with} \hspace{2mm} \gamma^1 < \frac{1}{L'\sqrt{d_hd_v}}
\end{equation}
the upper bound of expected gradients in (\ref{eq:upperbound}) decreases monotonically. Further, if the sequence for $\gamma^k$ satisfies 
 \begin{equation}\label{eq:gammabounds}\begin{aligned}
  &\lim_{N \to \infty} \sum_{k=1}^{N} \gamma^{k} \to \infty \hspace{1mm} , \hspace{1mm} \lim_{N \to \infty} \sum_{k=1}^{N} (\gamma^{k})^2 < \infty \hspace{4mm} \\  
  &\text{then} \hspace{2mm} \lim_{N \to \infty} \mathbb{E} (\| \nabla_{\bf W} f({\bf W}^R)\|^2) \to 0
 \end{aligned}\end{equation}
\end{lemma}
\begin{proof} 
We first show the monotonicity of the expected gradients followed by its limiting behavior.
Observe that whenever $\gamma^k < \frac{1}{L'\sqrt{d_hd_v}}$, we have 
\begin{equation*} 
\biggl( 2 - L'\sqrt{d_hd_v}\gamma^k \biggr) > 1 \hspace{2mm} \forall \hspace{2mm} k
\end{equation*}
Then the upper bound in \eqref{eq:upperbound} reduces to
\begin{equation} \label{eq:upperbound1}
\mathbb{E} (\| \nabla_{\bf W} f({\bf W}^R)\|^2) \leq \frac{D_f + (\sqrt{d_hd_v})^{3}L^2L' \sum_{k=1}^N(\gamma^k)^2 }{\sum_{k=1}^N \gamma^k} 
\end{equation} 
To show that right hand side in the above inequality decreases as $N$ increases, we need to show the following
\begin{equation}\label{eq:monoclaim}
\frac{D_f + (\sqrt{d_hd_v})^{3}L^2L' \sum_{k=1}^{N+1}(\gamma^k)^2 }{\sum_{k=1}^{N+1} \gamma^k} \leq \frac{D_f + (\sqrt{d_hd_v})^{3}L^2L' \sum_{k=1}^N(\gamma^k)^2 }{\sum_{k=1}^N \gamma^k}
\end{equation}
By denoting the terms in the above inequality as follows,
\begin{equation}\label{eq:terms} \begin{split}
a = D_f + (\sqrt{d_hd_v})^{3}L^2L' &\sum_{k=1}^N(\gamma^k)^2 \\
b = (\sqrt{d_hd_v})^{3}L^2L' (\gamma^{N+1})^2 \hspace{5mm} c &= \sum_{k=1}^N \gamma^k \hspace{5mm} d = \gamma^{N+1} \\
\end{split} \end{equation}
To show that the inequality in Equation \ref{eq:monoclaim} holds,
\begin{equation}\label{eq:monoclaim2} \begin{split}
\frac{a+b}{c+d} &\leq \frac{a}{c} \iff b \leq \frac{a}{c}d \\
&\iff (\sqrt{d_hd_v})^{3}L^2L' (\gamma^{N+1})^2 \leq \frac{D_f + (\sqrt{d_hd_v})^{3}L^2L'\sum_{k=1}^N(\gamma^k)^2}{\sum_{k=1}^N \gamma^k} \gamma^{N+1} \\
\end{split} \end{equation}
Rearranging the terms in the last inequality above, we have
\begin{equation}\label{eq:monoclaim3} \begin{split}
&(\sqrt{d_hd_v})^{3}L^2L' \sum_{k=1}^N \gamma^{N+1}\gamma^{k} \leq 
D_f + (\sqrt{d_hd_v})^{3}L^2L'\sum_{k=1}^N(\gamma^k)^2 \\
& \iff 
(\sqrt{d_hd_v})^{3}L^2L' \sum_{k=1}^N \gamma^{k} (\gamma^{N+1} - \gamma^k) \leq D_f \\
\end{split} \end{equation}
Recall that $D_f = 2(f({\bf W}^0) - f^{*})$; so without loss of generality we always have $D_f \geq 0$. With this result, the last inequality in \eqref{eq:monoclaim3} is always satisfied whenever 
$\gamma^{N+1} \leq \gamma^{k}$ for $k = 1,\ldots,N$. 
Since this needs to be true for all $N$, require $\gamma^{k+1} \leq \gamma^{k}$ for $k = 1,\ldots,N-1$. 
This proves the monotonicity of expected gradients. For the limiting case, recall the relaxed upper bound from \eqref{eq:upperbound1}. 
Whenever $\lim_{N \to \infty} \sum_{k=1}^{N} \gamma^k \to \infty \hspace{1mm} , \hspace{1mm} \lim_{N \to \infty} \sum_{k=1}^{N} (\gamma^k)^2 < \infty$, the right hand side in \eqref{eq:upperbound1} converges to $0$.
\end{proof}

The second part of the lemma is easy to ensure by choosing diminishing step-sizes (as a function of $k$). This result ensures the convergence of expected gradients, 
provides an easy way to construct $\mathbb{P}_R(\cdot)$ based on (\ref{eq:pmf}) and (\ref{eq:gammabounds}), and to decide the stopping iteration based on $\mathbb{P}_R(\cdot)$ ahead of time.

{\em Remarks.} Note that the maximum $\gamma^k$ in (\ref{eq:gammabounds}) needed to ensure the monotonic decrease of expected gradients depends on $L'$.
Whenever the estimate of $L'$ is too loose, the corresponding $\gamma^k$ might be too small to be practically useful.
An alternative in such cases is to compute the RSG updates for some $N$ (fixed a priori) iterations using a reasonably small stepsize, and select $R$ to be the iteration with the smallest possible gradient $\| \nabla_{\bf W} f({\bf W}^k)\|^2$ or the cumulative gradient $\sum_{i=1}^{k}\| \nabla_{\bf W} f({\bf W}^i)\|^2$ among some last $N_1<N$ iterations. 
While a diminishing stepsize following (\ref{eq:gammabounds}) is ideal, the next result gives the best possible {\em constant} stepsize $\gamma = \gamma^k, \forall k$, and the corresponding rate of convergence.

\vspace{2mm}
\begin{corollary}[\bf Convergence of one--layer DA]\label{thm:conv}
The optimal constant step sizes $\gamma^k$ are given by 
\begin{equation}\label{eq:optstep}
\gamma^k = \frac{D}{\sqrt{N}(d_hd_v)^{3/4}} \hspace{3mm} \forall k; \hspace{2mm} 0 < D \leq \frac{\sqrt{N}}{L'}(d_hd_v)^{1/4}
\end{equation}
If we denote $\bar{D} = \frac{D_f}{D} + DL^2L'$, then we have
\begin{equation}\label{eq:conv}
\mathbb{E} (\| \nabla_{\bf W} f({\bf W}^R)\|^2) \leq \bar{D} \frac{(d_hd_v)^{3/4}}{\sqrt{N}} \end{equation}
\end{corollary}
\begin{proof} 
Using constant stepsizes $\gamma^{k} = \gamma, k = 1,\ldots,N$, the convergence bound in \eqref{eq:upperbound} reduces to
\begin{equation}\label{eq:reduced} 
\mathbb{E} (\| \nabla_{\bf W} f({\bf W}^R)\|^2) \leq \frac{D_f + (\sqrt{d_hd_v})^{3}L^2L'N(\gamma)^2 }{N\gamma(2 - L'\sqrt{d_hd_v}\gamma)} 
\end{equation}
To achieve monotonic decrease of expected gradients, we require $\gamma^k < \frac{1}{L'\sqrt{d_hd_v}}$ (from \eqref{eq:gammabounds} in Lemma \ref{thm:convexpgrad}). 
For such $\gamma^k$s,  
\begin{equation*} 
\biggl( 2 - L'\sqrt{d_hd_v}\gamma^k \biggr) > 1 \hspace{2mm} \forall \hspace{2mm} k
\end{equation*}
which when used in \eqref{eq:reduced} gives,
\begin{equation}\label{eq:reduced1}
\mathbb{E} (\| \nabla_{\bf W} f({\bf W}^R)\|^2) \leq \frac{D_f}{N\gamma} + (\sqrt{d_hd_v})^{3}L^2L'\gamma
\end{equation}
Observe that as $\gamma$ increases (resp. decreases), the two terms on the right hand side of above inequality decreases (resp. increases) and increase (resp. decreases).
Therefore, the optimal $\gamma = \gamma^k$ for all $k$, is obtained by balancing these two terms, as in
\begin{equation}\label{eq:balance}
\frac{D_f}{N\gamma} = (\sqrt{d_hd_v})^{3}L^2L'\gamma \implies \gamma = \gamma^k = \frac{\sqrt{D_f}}{\sqrt{NL^2L'}(d_hd_v)^{3/4}}
\end{equation}
However, the above choice of $\gamma^k$ has the unknowns $D_f$, $L'$ and $L^2$ (although note that the later two constants can be empirically estimated by sampling the loss functions $\mathcal{L}(\cdot)$ for different choices of $x$ and ${\bf W}$).
Replacing $\sqrt{\frac{D_f}{L'L^2}}$ by some $D$, the best possible choice constant stepsize is 
\begin{equation}\label{eq:gammak}
\gamma = \gamma^k = \frac{D}{\sqrt{N}(d_hd_v)^{3/4}} \forall \hspace{2mm} k
\end{equation}
Since $\gamma^k$ needs to be smaller than $\frac{1}{L'\sqrt{d_hd_v}}$ as discussed at the start of the proof, we have 
\begin{equation}\label{eq:Dbound}
\frac{D}{\sqrt{N}(d_hd_v)^{3/4}} < \frac{2}{L'\sqrt{d_hd_v}} \implies  D \leq \frac{2\sqrt{N}}{L'}(d_hd_v)^{1/4}
\end{equation}
Now substituting this optimal constant stepsize from \eqref{eq:gammak} into the upper bound in \eqref{eq:reduced1} we get
\begin{equation}\label{eq:reduced2} \begin{aligned}
\mathbb{E} (\| \nabla_{\bf W} f({\bf W}^R)\|^2) & \leq \frac{D_f}{N\gamma} + (d_hd_v)^{3/2}L^2L'\gamma \\
& = \frac{D_f(d_hd_v)^{3/4}}{\sqrt{N}D} + \frac{DL^2L'(d_hd_v)^{3/4}}{\sqrt{N}} \\
\end{aligned} \end{equation}
and by denoting $\bar{D} = \frac{D_f}{D} + DL^2L'$, we finally have
\begin{equation}\label{eq:finalconv}
\mathbb{E} (\| \nabla_{\bf W} f({\bf W}^R)\|^2) \leq \bar{D} \frac{(d_hd_v)^{3/4}}{\sqrt{N}}
\end{equation}
\end{proof}

The upper bound in (\ref{eq:upperbound}) can be written as a summation of two terms, one of which involves $D_f$. 
The optimal stepsize in (\ref{eq:optstep}) is calculated by balancing these terms as $N$ increases (refer to the supplement).
The ideal choice for $D$ is $\sqrt{\frac{D_f}{L^2L'}}$ in which case $\bar{D}$ reduces to $2\sqrt{D_fL'L^2}$.
For a fixed network size ($d_h$ and $d_v$), Corollary \ref{thm:conv} shows that the rate of convergence for one--layer DA pre-training using RSG is $\mathcal{O}(1/\sqrt{N})$.
It is interesting to see that the convergence rate is proportional to $(d_hd_v)^{3/4}$ where the number of parameters of our bipartite network (of which DA is one example) is $d_hd_v$.

Corollary \ref{thm:conv} gives the convergence properties of a single RSG run over some $R$ iterations. 
However, in practice one is interested in a large deviation bound, where the best possible solution is selected from multiple independent runs of RSG. 
Such a large deviation estimate is indeed more meaningful than one RSG run because of the randomization over $\eta$ in (\ref{eq:daminrsg}). 
Consider a $C$--fold RSG with $C\geq 1$ independent RSG estimates of ${\bf W}$ denoted by ${\bf W}^{R_1},\ldots,{\bf W}^{R_C}$.
Using the expected convergence 
from (\ref{eq:conv}), we can compute a $(\epsilon,\delta)$-solution defined as,

\vspace{2mm}
\begin{definition}[\bf $(\epsilon,\delta)$-solution]
For some given $\epsilon>0$ and $\delta \in (0,1)$, an $(\epsilon,\delta)$-solution of one--layer DA is given by $\{{\bf W^{R_c}}\}, c=1,\ldots,C$ such that
\begin{equation} \label{eq:epsdel}
Pr\left( \min_{{1,\ldots,C}} \| \nabla_{\bf W} f({\bf W}^{R_c}) \|^2 \geq \epsilon  \bar{D} \right) \leq \delta
\end{equation}
\end{definition}
$\epsilon$ governs the goodness of the estimate ${\bf W}$, and $\delta$ bounds the probability of good estimates over multiple independent RSG runs.
Since $N$ is the maximum iteration count (i.e., maximum number of $\mathcal{SFO}$ calls), the number of data instances required is $S = N/t$, where $t$ denotes the average number of 
times each instance is used by the oracle. 
Although in practice there is no control over $t$ (in which case, we simply have $S \leq N$), we estimate the required sample size and the minimum number of folds ($C$) in terms of $t$, as shown by the following result.

\vspace{2mm}
\begin{corollary}[\bf Sample size estimates of one--layer DA]\label{thm:samp} 
The number of independent RSG runs ($C$) and the number of data instances ($S$) required to compute a $(\epsilon,\delta)$-solution are given by
\begin{equation}\label{eq:samp}
C(r,\delta) \geq \ceil[\bigg]{\frac{\log(\frac{1}{\delta})}{\log(\sqrt{r})}} \hspace{2mm} ; \hspace{2mm}
S(r,\epsilon) \geq \frac{r(d_hd_v)^{3/2}}{t\epsilon^2} 
\end{equation} 
where $r > 1$ is a given constant, $\ceil[\big]{\cdot}$ denotes ceiling operation and $t$ denotes the average number of times each data instance is used. 
\end{corollary}
\begin{proof} 
Recall that a $(\epsilon,\delta)$-solution is defined such that 
\begin{equation}\label{eq:epsdel}
Pr\left( \min_{{1,\ldots,C}} \| \nabla_{\bf W} f({\bf W}^{R_c}) \|^2 \geq \epsilon  \bar{D} \right) \leq \delta
\end{equation}
for some given $\epsilon>0$ and $\delta \in (0,1)$. 
Using basic probability properties, 
\begin{equation}\label{eq:epsdelcomp} \begin{aligned}
Pr\left( \min_{{1,\ldots,C}} \| \nabla_{\bf W} f({\bf W}^{R_c}) \|^2 \geq \epsilon  \bar{D} \right) & = Pr\left( \| \nabla_{\bf W} f({\bf W}^{R_c}) \|^2 \geq \epsilon  \bar{D} \hspace{2mm} \forall \hspace{2mm} c = 1,\ldots,C \right) \\
& = \prod_{c=1}^{C} Pr\left( \| \nabla_{\bf W} f({\bf W}^{R_c}) \|^2 \geq \epsilon  \bar{D} \right) \\
\end{aligned} \end{equation}
Using Markov inequality and \eqref{eq:conv}, 
\begin{equation}\label{eq:epsdelcomp2} \begin{aligned}
Pr\left( \| \nabla_{\bf W} f({\bf W}^{R_c}) \|^2 \geq \epsilon  \bar{D} \right) & \leq \frac{\mathbb{E} (\nabla_{\bf W} f({\bf W}^{R_c}) \|^2)}{\epsilon  \bar{D}} \\
& \leq \frac{(d_hd_v)^{3/4}}{\epsilon\sqrt{N}} \\
\end{aligned} \end{equation}
Hence, the number of $\mathcal{SFO}$ calls per RSG is at least $N > \frac{(d_hd_v)^{3/2}}{\epsilon^2}$ for the above probability to make sense. 
If $r>1$ is a constant, then the number of calls per RSG is $N = \frac{r(d_hd_v)^{3/2}}{\epsilon^2}$. 
Using this identity, and \eqref{eq:epsdelcomp} and \eqref{eq:epsdelcomp2}, we get
\begin{equation}\label{eq:epsdelcomp3}
  Pr\left( \min_{{1,\ldots,C}} \| \nabla_{\bf W} f({\bf W}^{R_c}) \|^2 \geq \epsilon  \bar{D} \right) \leq \prod_{c=1}^{C} \frac{1}{\sqrt{r}} = \frac{1}{r^{C/2}}
\end{equation}
To ensure that this probability is smaller than a given $\delta$ and noting that $C$ is a positive integer, we have 
\begin{equation}\label{eq:epsdelcompfin}
\frac{1}{r^{C/2}} \leq \delta \implies C(r,\delta) \geq \frac{log(\frac{1}{\delta})}{log(\sqrt{r})} := \ceil[\bigg]{\frac{log(\frac{1}{\delta})}{log(\sqrt{r})}} 
\end{equation}
where $\ceil[\big]{\cdot}$ denotes ceiling operation. 
Note that there is no randomization over the data instances among multiple instances of RSG ($c = 1,\ldots,C$).
That is, each RSG is going to use all the available data instances. Hence, we can just look at one RSG to derive the sample size required.
Let $S$ be the number of data instances, $t_s$ be the number of times $s^{th}$ instance is used in one RSG and $t = \mathbb{E}(t_s)$ be the average number of times each instance/example is used.
We then have
\begin{equation}\label{eq:sampcomp}
  N = \sum_{s=1}^{S} t_{s} \implies N \approx \mathbb{E}(N) = S\mathbb{E}(t_s) \implies S = \frac{N}{t} \geq \frac{r(d_hd_v)^{3/2}}{t\epsilon^2}
\end{equation}
\end{proof}

The above result shows that the required sample size is $\mathcal{O}(1/\epsilon^2)$, which is easy to see from the convergence rate of $\mathcal{O}(1/\sqrt{N})$ in (\ref{eq:conv}).
The constant $r$ in Corollary \ref{thm:samp} acts like a trade--off parameter between the number of folds $C(r,\delta)$ and the sample size $S(r,\epsilon)$. 
Hence, not surprisingly, more folds are needed to guarantee a $(\epsilon,\delta)$-solution for a smaller $S$.
Note that the minimum possible $S$ is $\frac{(d_hd_v)^{3/2}}{t\epsilon^2}$, below this quantity the idea of computing an $(\epsilon,\delta)$-solution in a large deviation sense is not meaningful (refer to proof of Corollary \ref{thm:samp} in the supplement).

{\em Remarks.} To get a practical sense of (\ref{eq:samp}), consider a DA with $d_v=100, d_h=20$.
According to Corollary \ref{thm:samp}, the number of data instances for computing a $(0.05,0.05)$-solution with $t=10^3$ is at least $0.3$ million. 
Depending on the structural characteristics of the data (variance of each dimension, correlations across multiple dimensions etc.), which we do not exploit, the bound from Corollary \ref{thm:samp} will overestimate the required number of samples, as expected.
Overall, the convergence and sample size bounds in (\ref{eq:conv}) and (\ref{eq:samp}) provide some justification of a behavior which is routinely observed in practice --- large number of unsupervised data instances are required for efficient pre-training of deep architectures (Chapter $4$, \cite{bengio2009learning},\cite{erhan2010does}). 
Note that the results in the convergence bound in (\ref{eq:conv}) do not differentiate between the visible and hidden layers, implying that the bound is symmetric with respect to $d_h$ and $d_v$.
However, there is empirical evidence that the choice of $d_h$ would affect the reconstruction error with oversized networks giving better generalization in general \cite{lawrence1998size, paugam1997size}.
This can be seen by recalling that until $d_h$ is more than the dimensionality of the low-dimensional manifold on which the input data lies, the DA setup may not be able to compute good estimates of ${\bf W}$. We discuss this issue in more detail when presenting our experiments in Section \ref{sec:exp}. 

Recall that pre-training is done layer-wise in deep architectures with multiple hidden layers. 
Hence, the bounds presented above in (\ref{eq:conv}) and (\ref{eq:samp}) directly apply to stacked DAs with no changes.
For stacked DAs the total number of $\mathcal{SFO}$ calls would simply be the sum of the calls across all the layers. 
The results also provide insights regarding convolutional neural networks where one-to-two layer neural nets are learned from small regions (e.g., local neighborhoods in imaging data), whose outputs are then combined using some nonlinear pooling operation \cite{lee2009convolutional}. 
Observe that the sub-linear dependence of convergence rate on the network size ($d_h.d_v$) from (\ref{eq:conv}) implies that whenever $S$ (and hence $N$) is reasonable large, small networks are learned efficiently. 
This partially supports the evidence that deep convolutional networks with multiple levels of pooling over large number of small networks are successful in learning complex concepts \cite{lee2009convolutional, krizhevsky2012imagenet}.  
With these results in hand, we now consider the case of distributed synchronous pre-training where small parts of the whole network are learned at-a-time. 

\section{Distributed DA pre-training}
\label{sec:dist}

The results in the previous section show that the convergence rate has polynomial dependence on the size of the network ($d_hd_v$), where the number of $\mathcal{SFO}$ calls increases as $(d_hd_v)^{3/2}$.
Although this is unlikely to happen in practice because of the redundancies across the input data dimensions (for example, sufficiently strong correlations across multiple input dimensions, presence of invariant dimensions etc.), the results in Corollaries \ref{thm:conv} and \ref{thm:samp} show that pre-training {\em very large} DAs is impractical with smaller sample sizes (and thereby fewer iterations). 
There is empirical evidence supporting that this is indeed the case in practice \cite{erhan2009difficulty, raina2009large}.
Several authors have suggested learning parts of the network instead.
Recently, \cite{dean2012large} showed empirical results on how distributed learning substantially improves convergence while not sacrificing test-time performance. 
Motivated by these ideas, we extend the results presented in Section \ref{sec:uppersamp} to the distributed pre-training setting.
We first show that the objective in (\ref{eq:damin}) lends itself to be distributed in a simple way where the whole network is broken down into multiple parts, and each such {\it sub-network} is learned in a synchronous manner.
By relating the corruption probabilities of these sub-networks to that of the parent DA, we compute a lower bound on the number of sub-networks required. 
Later, we present the convergence and sample size results for this distributed DA pre-training setting. 

Recall that the objective of DA in (\ref{eq:damin}) involves an expectation over corruptions $\tilde{\bf x}$ where certain visible units are nullified (set to $0$).
This implies that the corrupted dimension does not provide any information to the hidden layer.
Since the DA network is bipartite, the objective can then be separated into sub-networks (referred to as sub--DAs) -- while the hidden layer remains unchanged, we use only a {\em subset} of all available $d_v$ visible units. 
For each such sub--DA of size ($\ceil{\tau d_v},d_h$), where $0 < \tau < 1$ is the fraction of the visible layer used, the inputs from all the left out ($d_v-\ceil{\tau d_v}$) visible units is zero i.e., their corruption probability is $1$.
Now, consider the setting where $B$ such sub--DAs constructed by sampling $\ceil{\tau d_v}$ number of visible units with replacement. 
The following result shows the equivalence of learning these $B$ sub--DAs to learning one large DA of size ($d_v,d_h$). 

\begin{lemma}[\bf Distributed learning of one--layer DA]\label{thm:dda}
  Consider a DA network of size ($d_v,d_h$) with corruption probability $\zeta$ and some $1-\zeta < \tau < 1$ and 
  $0 < \phi \ll 1$. 
  Learning this DA is equivalent to learning $B > \frac{\log(\phi)}{\log(1-\tau)}$ number of DAs of size ($\ceil{\tau d_v},d_h$) with corruption probability 
  $1 - \frac{1-\zeta}{\tau}$, whose visible units are a fraction $\tau$ (with replacement) of the total available $d_v$ units, where $\ceil{\cdot}$ denotes the ceiling operation. 
\end{lemma}
\begin{proof} 
Recall that the DA objective is
\begin{equation}\label{eq:damin}
\min_{{\bf W}} \hspace{3mm} \mathbb{E}_{p({\bf x},{\bf \tilde{x}})} \| {\bf x} - \sigma({\bf W}^T \sigma({\bf W\tilde{x}})) \|^{2}
\end{equation} 
By considering one term from this expectation, we show that it is equivalent to learning two disjoint DAs of sizes ($\ceil{\tau d_v},d_h$) and ($d_v-\ceil{\tau d_v},d_h$) synchronously.
Without loss of generality, let this term correspond to the last $d_v-\ceil{\tau d_v}$ visible units be corrupted with probability $1$ i.e., are set to $0$. 
For the rest of the proof, any visible unit that is set to $0$ via corruption will be referred to as a `clamped' unit.

Let $W_1$ (of size $d_h \times \ceil{\tau d_v}$) and $W_2$ (of size $d_h \times d_v-\ceil{\tau d_v}$) be the matrices of edge weights (i.e., unknown parameters) from the un-clamped and clamped visible units to all $d_h$ hidden units respectively.
For some inputs ${\bf x}$, let ${\bf x}_1$ (of length $\ceil{\tau d_v} \times 1$) and ${\bf x}_2$ ($d_v-\ceil{\tau d_v} \times 1$) be the un-clamped and clamped parts.
Hence $\bar{\bf x}_1 = {\bf x}_1$ and $\bar{\bf x}_2 = 0$.
Then the hidden activation ${\bf h}$, and the corresponding un-clamped and clamped reconstructions, $\hat{\bf x}_1$ and $\hat{\bf x}_2$ have the following structure,
\begin{equation}\label{eq:clamping} \begin{aligned}
{\bf h} = \sigma (W_1 \bar{\bf x}_1 + W_2 0) = \sigma (W_1 \bar{\bf x}_1) &\hspace{5mm}
\hat{\bf x}_1 = \sigma (W_1^T \sigma (W_1 \bar{\bf x}_1)) \\ 
\hat{\bf x}_2 = \sigma (W_2^T \sigma (W_1 \bar{\bf x}_1)) &= \sigma (W_2^T \sigma (W_1 \bar{\bf x}_1 + W_2 \bar{\bf x}_2))
\end{aligned} \end{equation}
The objective for the term considered then simplifies to 
\begin{equation}\label{eq:clampingobj}
\| {\bf x} - \hat{\bf x} \|^2 = \| {\bf x}_1 - \sigma (W_1^T \sigma (W_1 \bar{\bf x}_1)) \|^2 + 
\| {\bf x}_2 - \sigma (W_2^T \sigma (W_1 \bar{\bf x}_1 + W_2 \bar{\bf x}_2)) \|^2
\end{equation}
It is easy to see that the first term from the above summation is exactly minimizing the recovery of ${\bf x}_1$ with no corruption applied to it. 
That is to say, it corresponds to one of the terms in the objective of a smaller DA of size $\ceil{\tau d_v},d_h$. 
The second term in the summation has similar structure however with an extra $W_1 \bar{\bf x}_1$ within the inner sigmoid. 
If $W_1$ is fixed, then this the second term is minimizing the recovery of ${\bf x}_2$ with `complete' corruption applied to all the $1-\ceil{\tau d_v}$ dimensions.
Hence we can first pre-train the ($\ceil{\tau d_v},d_h$) sized sub-DA, and use the learned $W_1$ as a constant bias, and then learn the ($d_v-\ceil{\tau d_v},d_h$) sized sub-DA. 
This strategy can be shown for all the terms in the objective in (\ref{eq:damin}). 
With this, we can begin with set of sub-DAs of size ($\ceil{\tau d_v},d_h$) each and pre-train then one at-a-time in a synchronous manner, thereby justifying the distributed setting for DA pre-training. 

Now consider such a setup where many such ($\ceil{\tau d_v},d_h$) sub--DAs are learned synchronously by randomly sampling 
different subsets $\ceil{\tau d_v}$ of the total available visible units. 
It is easy to see that, in expectation this sequential distributed learning is equivalent to minimizing all the terms inside the expectation in \eqref{eq:damin}.
Hence learning the big ($d_v,d_h$) DA is the same as sequentially learning small DAs of size ($\ceil{\tau d_v},d_h$) where the units $\ceil{\tau d_v}$ are chosen at random.
In practice, this is achieved only if each of the visible unit is included in at least one of the sub--DAs (i.e., all unknown parameters are updated at least once).
Let $B$ be the number of sub--DAs that are learned sequentially. 
If $0 < \phi \ll 1$ denotes the probability that a given unit is not in all the $B$ sub--DAs (ideally, $\phi$ should be small in practice). 
Then, it is easy to see that this probability is given by $(1-\tau)^B$ because the probability that a particular unit is sampled to be included in one sub--DA is $\tau$. 
Since $1-\tau < 1$, we then have
\begin{equation}\label{eq:Bbound}
(1-\tau)^B \ll \phi \implies B \gg \frac{\log(\phi)}{\log(1-\tau)}
\end{equation}

We now relate the corruption probabilities of the sub--DAs (denoted by $q$) to that of the mother DA ($\zeta$). 
Recall that clamping is the same as corrupting (i.e., setting the input from that unit to be $0$).
Given the sampling fraction $\tau$, the probability that a given visible unit ($1,\ldots,d_v$) belongs to one sub--DA is $\tau$.
Further, if $q$ is the corruption probability of this sub--DA, then the un-clamping probability of a given unit is $(1-\tau)+\tau q$.  
If the $B$ sub--DAs are constructed independently by sampling the visible units with replacement, then the overall corrupting (un-clamping) probability is 
$\frac{(B-\tau B)+\tau qB}{B} = 1-\tau +\tau q$.
We require this to be equal to $\zeta$, which then gives $q = 1 - \frac{1-\zeta}{\tau}$ (with $1-\zeta < \tau < 1$).
\end{proof}

The above statement (proof in supplement) establishes the equivalence of distributed DA (dDA) pre-training to the non-distributed case by explicitly considering the DA's property of using nullified/corrupted inputs (which then provide no new information to the objective). 
We remark that Lemma \ref{thm:dda} is specific for the case of DAs and (unlike many other results in this paper) may not be directly applicable for other types of auto-encoders that do not involve an explicit corruption function.
Also, $\tau$ and $\zeta$ should be chosen carefully so that $1-\zeta < \tau$ and $1 - \frac{1-\zeta}{\tau}$ does not end up too close to $1$.
Specifically, whenever $\zeta$ is very small, according to Lemma \ref{thm:dda}, there is very little room for distribution because $\tau$ will be close to $1$.
This is not surprising because, with small $\zeta$, the DA is allowed to discard visible units very rarely, pushing $\tau$ closer to $1$, where the distributed setup tends to behave like the non-distributed case. 
Although these requirements seem too restrictive, we show in Section \ref{sec:exp} that they can be fairly relaxed in practice.
Overall, Lemma \ref{thm:dda} provides some justification (from the perspective of the autoencoder design itself) for distributing the learning process.   
The lower bound on $B$ in Lemma \ref{thm:dda} ensures that all the unknown parameters are updated in at least one of the $B$ sub--DAs. 
Hence, in practice, $\phi$ can be chosen to be very small and the sub-DAs can be explicitly sampled to be ``non-overlapping'' (i.e., disjoint with respect to the parameters). 
Once the hyper-parameters $\tau$, $\zeta$ and $B$ are fixed, the recipe is simple. 
The dDA pre-training setup will involve running $B$ individual RSGs on randomly sampled disjoint sub-DAs. 
The $B$ sub-DAs share a common parameter set which holds the latest estimates of ${\bf W}$. 

Similar to the multi-fold RSG setup in Section \ref{sec:uppersamp}, we perform $M$ meta-iterations of the dDA pre-training, where each meta--iteration involves learning $B$ number of sub--DAs.
Because sub--DAs are constructed randomly, different meta-iterations end up with different set of sub--DAs, ensuring low variance in the estimate of ${\bf W}$ corresponding to the $(\epsilon,\delta)$-solution (\ref{eq:epsdel}).
It is clear that due to the reduction in the size of the network by a factor of $\tau$, the convergence rate and required sample sizes (see (\ref{eq:conv}) and (\ref{eq:samp})) will improve in this distributed case.
This observation is formalized in the two results below. 
Here, $\gamma^b_k$ denotes the step size in $k^{th}$ meta--iteration for $b^{th}$ RSG (corresponding to $b^{th}$ sub--DA) and $N$ is the number of $\mathcal{SFO}$ calls for each of the $B$ RSGs.  
The subscript $b$ in ${\bf W}^{R_b}_b$ represents the updates of $b^{th}$ RSG where $R_b$ is its stopping iteration.

\begin{corollary}[\bf Convergence of one--layer dDA]\label{thm:convdda}
  The optimal constant step size $\gamma_b^k$ is given by
{\small \begin{equation}\label{eq:optstepdda}
  \gamma_b^k = \frac{D}{\sqrt{N}(\tau d_hd_v)^{3/4}} \hspace{3mm} \forall b,k; \hspace{2mm} 0 < D \leq \frac{\sqrt{N}}{L'}(\tau d_hd_v)^{1/4}
\end{equation}}
By selecting $B$ according to Lemma \ref{thm:dda}, and denoting $\bar{D} = \frac{D_f}{BD} + DL^2L'$, we have,
\begin{equation}\label{eq:convdda} 
\mathbb{E} (\| \nabla_{\bf W} f({\bf W}^{R_b}_b)\|^2) \leq \bar{D} \frac{(\tau d_hd_v)^{3/4}}{\sqrt{N}} \end{equation}
\end{corollary}
\begin{proof} 
The proof for this theorem emulates the proofs of Lemma \ref{thm:expgrad} and Corollary \ref{thm:conv}. 
First we derive an upper bound on the expected gradients similar to the one in \eqref{eq:upperbound} of Lemma \ref{thm:expgrad}. 
Using this bound, we then compute the optimal stepsizes and the rate of convergence. 

In the distributed setting, we have $B$ number of RSGs running synchronously (or sequentially) and the size of each of the $B$ sub--DAs is ($\ceil{\tau d_v},d_h$).
This is the same as a ($d_v,d_h$) DA with $\tau d_hd_v$ unknowns (for notational convenience the ceil operator $\ceil{\cdot}$ is dropped in the analysis).
So, the bounds on the variance of noisy gradients ($G(\eta^k;{\bf W}^k)$) and the Lipschitz continuity of $f({\bf W})$ change as follows,
\begin{equation} \begin{aligned}\label{eq:varlipboundsdda}
\Var(G(\eta^k;{\bf W}^k)) &\leq \tau d_hd_vL^2 \\
\| \nabla_{{\bf W}} f({\bf W}) - \nabla_{{\bf W}} f(\hat{{\bf W}}) \| &\leq \sqrt{\tau d_hd_v}L' \|{\bf W} - \hat{{\bf W}}\| \\
\end{aligned} \end{equation}

We then have the following inequality for each of the $B$ RSGs based on the analysis in the proof of Lemma 3.2 until \eqref{eq:sumineq0}
\begin{equation} \begin{aligned}\label{eq:sumineq0dda}
\sum_{k=1}^{N_b} \left( \gamma_b^k - \frac{\sqrt{\tau d_hd_v}L'}{2}(\gamma_b^k)^2 \right) \|f({\bf W}_b^{k})\|^2 &\leq f({\bf W}_b^1) - f({\bf W}_b^{N_b+1}) \\
- \sum_{k=1}^{N_b} \biggl( \gamma_b^k - \sqrt{\tau d_hd_v}L'(\gamma_b^k)^2 \biggr) &\langle \nabla_{\bf W} f({\bf W}_b^{k}),\delta_b^k \rangle 
+ \frac{\sqrt{\tau d_hd_v}L'}{2} \sum_{k=1}^{N_b} (\gamma_b^k)^2 \|\delta_b^k\|^2 \\
\end{aligned} \end{equation}
It should be noted that the subscript $b$ indicates $b^{th}$ RSG i.e., ${\bf W}_b^k$ is the $k^{th}$ update from $b^{th}$ RSG.
$N_b$ denotes the maximum number of iterations of $b^{th}$ RSG.
Although the size of ${\bf W}$ is $d_h \times d_v$, only $\tau d_hd_v$ of the total $d_hd_v$ are being updated within a single RSG. 

Now recall that the sequential nature of the $B$ RSGs implies that the estimate of ${\bf W}$ at the end of $b^{th}$ RSG will be the starting point for the $(b+1)^{th}$ RSG. 
This implies that $f({\bf W}_b^{N_b+1}) = f({\bf W}_{b+1}^1)$ for all $b = 1,\ldots,B$.
Using this fact, we can then sum up all the $B$ inequalities of the form in \eqref{eq:sumineq0dda} to get,
\begin{equation}\label{eq:sumineq1dda} \begin{aligned}
\sum_{b=1}^{B} \sum_{k=1}^{N_b} \left( \gamma_b^k - \frac{\sqrt{\tau d_hd_v}L'}{2}(\gamma_b^k)^2 \right) \|f({\bf W}_b^{k})\|^2 &\leq f({\bf W}_1^1) - f({\bf W}_B^{N_B+1}) \\
- \sum_{b=1}^{B} \sum_{k=1}^{N_b} ( \gamma_b^k - \sqrt{\tau d_hd_v}L'(\gamma_b^k)^2 ) &\langle \nabla_{\bf W} f({\bf W}_b^{k}),\delta_b^k \rangle + \frac{\sqrt{\tau d_hd_v}L'}{2} \sum_{b=1}^{B} \sum_{k=1}^{N_b} (\gamma_b^k)^2 \|\delta_b^k\|^2 \\
\end{aligned} \end{equation}
Using the fact that $f^{*} \leq f({\bf W}_B^{N_B+1})$, we then have
\begin{equation}\label{eq:sumineq2dda} \begin{aligned}
\sum_{b=1}^{B} \sum_{k=1}^{N_b} \left( \gamma_b^k - \frac{\sqrt{\tau d_hd_v}L'}{2}(\gamma_b^k)^2 \right) \|f({\bf W}_b^{k})\|^2 &\leq f({\bf W}_1^1) - f^{*} \\
- \sum_{b=1}^{B} \sum_{k=1}^{N_b} ( \gamma_b^k - \sqrt{\tau d_hd_v}L'(\gamma_b^k)^2 ) &\langle \nabla_{\bf W} f({\bf W}_b^{k}),\delta_b^k \rangle + \frac{\sqrt{\tau d_hd_v}L'}{2} \sum_{b=1}^{B} \sum_{k=1}^{N_b} (\gamma_b^k)^2 \|\delta_b^k\|^2 \\
\end{aligned} \end{equation}

We now take the expectation of the above inequality over all the random variables involved in the $B$ RSGs, which include, the $B$ number of stopping criterions $R_b, b = 1,\ldots,B$ and the random processes $\eta_b^{[N_b]}, b = 1,\ldots,B$ ($\eta_b^{[N_b]}$ is the random process of $\eta$ within $b^{th}$ RSG).
First, note the following observations about $\langle \nabla_{\bf W} f({\bf W}_b^{k}),\delta_b^k \rangle$ and $\|\delta_b^k\|^2$
\begin{equation}\label{eq:ddaobs}\begin{split}
\mathbb{E}_{\eta_b^{[N_b]}} (G(\eta_b^k;{\bf W}_b^k)) = \nabla_{\bf W} f({\bf W}_b^{k})) & \implies \langle \nabla_{\bf W} f({\bf W}_b^{k}),\delta_b^k \rangle = 0 \\
\mathbb{E}_{\eta_b^{[N_b]}} \| \delta_b^k \|^2 = Var(G(\eta_b^k;{\bf W}_b^k)) & \leq \tau d_hd_vL^2
\end{split} \end{equation}
which follow from \eqref{eq:varlipboundsdda}. This implies that after taking the expectation of the inequality in \eqref{eq:sumineq2dda}, the last two terms on the right hand side will be,
\begin{equation}\label{eq:term1dda} \begin{aligned}
\mathbb{E}_{\eta^{[N]}} & \left[ \sum_{b=1}^{B} \sum_{k=1}^{N_b} \biggl( \gamma_b^k - \sqrt{\tau d_hd_v}L'(\gamma_b^k)^2 \biggr) \langle \nabla_{\bf W} f({\bf W}_b^{k}),\delta_b^k \rangle \right] \\
& \hspace{5mm} = \sum_{b=1}^{B} \sum_{k=1}^{N_b} \biggl( \gamma_b^k - \sqrt{\tau d_hd_v}L'(\gamma_b^k)^2 \biggr) \hspace{2mm} \mathbb{E}_{\eta^{[N_b]}} (\langle \nabla_{\bf W} f({\bf W}_b^{k}),\delta_b^k \rangle | \eta_b^{1},\ldots,\eta_b^{k}) = 0 \\
\end{aligned} \end{equation}
\begin{equation}\label{eq:term2dda} \begin{aligned}
\mathbb{E}_{\eta^{[N]}} \left[ \frac{\sqrt{\tau d_hd_v}L'}{2} \sum_{b=1}^{B} \sum_{k=1}^{N_b} (\gamma_b^k)^2 \|\delta_b^k\|^2 \right] &= \frac{\sqrt{\tau d_hd_v}L'}{2} \sum_{b=1}^{B} \sum_{k=1}^{N_b} (\gamma_b^k)^2 \mathbb{E}_{\eta^{[N_b]}} (\|\delta_b^k\|^2) \\
&\leq \frac{(\sqrt{\tau d_hd_v})^3L'L^2}{2} \sum_{k=1}^{N_b} (\gamma_b^k)^2 
\end{aligned} \end{equation}
where $\eta^{[N]}$ denotes the composition of the $B$ random processes $\eta^{[N_b]}, b = 1,\ldots,B$. 
Using \eqref{eq:term1dda} and \eqref{eq:term2dda} and \eqref{eq:sumineq2dda}, we get
\begin{equation}\label{eq:sumineq3dda} \begin{aligned}
\sum_{b=1}^{B} \sum_{k=1}^{N_b} &\left( 2\gamma_b^k - \sqrt{\tau d_hd_v}L'(\gamma_b^k)^2 \right) \mathbb{E}_{\eta^{[N]}} \|f({\bf W}_b^{k})\|^2 \\
&\leq 2(f({\bf W}_1^1) - f^{*}) + (\sqrt{\tau d_hd_v})^3L'L^2 \sum_{b=1}^{B} \sum_{k=1}^{N_b} (\gamma_b^k)^2 
\end{aligned} \end{equation}
Recall the definition of $\mathbb{P}_R(k)$ from \eqref{eq:pmf} in Lemma \ref{thm:expgrad}, which is
\begin{equation} \label{eq:pmfdda}
\mathbb{P}_R(k) = Pr(R = k) = \frac{2\gamma^k - L'\sqrt{d_hd_v}(\gamma^k)^2}{\sum_{k=1}^{N} 2\gamma^k - L'\sqrt{d_hd_v}(\gamma^k)^2} \hspace{3mm} k=1,\ldots,N
\end{equation}
Adapting this to the current case of $B$ sequential RSGs, we get
\begin{equation} \label{eq:pmfdda1}
\mathbb{P}_{R_b}(k) = Pr(R_b = k) := \frac{2\gamma_b^k - L'\sqrt{\tau d_hd_v}(\gamma_b^k)^2}{\sum_{b=1}^{B} \sum_{k=1}^{N_b} 2\gamma_b^k - L'\sqrt{\tau d_hd_v}(\gamma_b^k)^2} \hspace{1mm} k=1,\ldots,N \hspace{1mm} b = 1,\ldots,B
\end{equation}
Using this distribution of stopping criterion and taking the expectation of \eqref{eq:sumineq3dda} with respect the set of random variables to $R_b, b = 1,\ldots,B$, we get
\begin{equation}\label{eq:finalineqdda} \begin{aligned}
    \mathbb{E} (\| \nabla_{\bf W} f({\bf W}_b^R)\|^2) &= \sum_{b=1}^{B} \sum_{k=1}^{N_b} \frac{(2\gamma_b^k - L'\sqrt{\tau d_hd_v}(\gamma_b^k)^2) \mathbb{E}_{\eta^{[N]}} (\| \nabla_{\bf W} f({\bf W}_b^R)\|^2)}{\sum_{b=1}^{B} \sum_{k=1}^{N} (2\gamma_b^k - L'\sqrt{\tau d_hd_v}(\gamma_b^k)^2)} \\
& \leq \frac{D_f + (\sqrt{\tau d_hd_v})^{3}L^2L' \sum_{b=1}^B \sum_{k=1}^N(\gamma_b^k)^2 }{\sum_{b=1}^B \sum_{k=1}^N (2\gamma_b^k - L'\sqrt{\tau d_hd_v}(\gamma_b^k)^2)} \\
\end{aligned} \end{equation}
Observe that whenever $B$ is selected as in Lemma \ref{thm:dda} with sufficiently high $\phi$, each of the $d_hd_v$ unknowns is updated in at least one of the $B$ RSGs.
Hence, all the unknowns are covered in the left hand side above, see \ref{eq:finalineqdda}.

We now compute the optimal stepsizes and the corresponding convergence rate using the upper bound in \eqref{eq:finalineqdda}.
At any given point of time only one of the $B$ RSGs will be running. 
So, using Lemma \ref{thm:conv}, the optimal constant stepsize for $b^{th}$ RSG is then given by  
\begin{equation*}
  \gamma_b^k = \gamma_b = \frac{D}{\sqrt{N_b}(\tau d_hd_v)^{3/4}} \hspace{2mm} \text{where} \hspace{2mm} D \leq \frac{\sqrt{N_b}}{L'}(\tau d_hd_v)^{1/4}
\end{equation*}
Assuming $N_b = N$ for all $b = 1,\ldots,B$, we then have
\begin{equation}\label{eq:gammakdda}
\gamma_b^k = \gamma = \frac{D}{\sqrt{N}(\tau d_hd_v)^{3/4}} \hspace{1mm} \forall k,b \hspace{2mm} \text{where} \hspace{2mm} D \leq \frac{\sqrt{N}}{L'}(\tau d_hd_v)^{1/4}
\end{equation}

With this in hand, we now derive the convergence rate. 
Using some constant stepsizes $\gamma_b^k = \gamma$ for all $k,b$ and assumption that $N_b = N$ for all $b$, the upper bound in \eqref{eq:finalineqdda} becomes
\begin{equation}\label{eq:finalineq2dda} \begin{split}
\mathbb{E} (\| \nabla_{\bf W} f({\bf W}_b^{R_b})\|^2) & \leq 
\frac{D_f + (\sqrt{\tau d_hd_v})^{3}L^2L' NB\gamma^2 }{NB(2\gamma - L'\sqrt{\tau d_hd_v}\gamma^2)} \\
& \leq \frac{D_f + (\sqrt{\tau d_hd_v})^{3}L^2L' NB\gamma^2 }{NB\gamma} \\
\end{split} \end{equation}
where the last inequality uses the fact that $(2 - L'\sqrt{\tau d_hd_v}\gamma) > 1$ (which follows from Lemma \ref{thm:convexpgrad} and was used in deriving the stepsizes in Lemma \ref{thm:conv}).
Substituting for $\gamma$ from \eqref{eq:gammakdda} in the above inequality gives, 
\begin{equation}\label{eq:reduced2dda} \begin{aligned}
\mathbb{E} (\| \nabla_{\bf W} f({\bf W}_b^{R_b})\|^2) & \leq \frac{D_f}{NB\gamma} + (\tau d_hd_v)^{3/2}L^2L'\gamma \\
& = \frac{D_f(\tau d_hd_v)^{3/4}}{\sqrt{N}DB} + \frac{DL^2L'(\tau d_hd_v)^{3/4}}{\sqrt{N}} \\
\end{aligned} \end{equation}
By denoting $\bar{D} = \frac{D_f}{BD} + DL^2L'$, we finally have
\begin{equation}\label{eq:finalconvdda}
\mathbb{E} (\| \nabla_{\bf W} f({\bf W}_b^{R_b})\|^2) \leq \bar{D} \frac{(\tau d_hd_v)^{3/4}}{\sqrt{N}}
\end{equation}
\end{proof}

\begin{corollary}[\bf Sample size estimates of one--layer dDA]\label{thm:sampdda}
  The number of meta--iterations ($M$) and the number of data instances ($S$) required to compute a $(\epsilon,\delta)$-solution in the distributed setting are 
  \begin{equation}\label{eq:sampdda}
    M(r,\delta) \geq \ceil[\bigg]{\frac{\log(\frac{1}{\delta})}{\log(\sqrt{r})}} \hspace{4mm} ; \hspace{4mm}
    S(r,\epsilon) \geq \frac{r(\tau d_hd_v)^{3/2}}{t\epsilon^2} 
  \end{equation} 
  where $r > 1$ is a given constant and $t$ denotes the average number of times each data instance is used within each sub--DA. \end{corollary}
  \begin{proof} 
First observe that there is no randomization of data instances across the $B$ sub---DAs. 
Hence we can compute the sample sizes $S$ from a single sub--DA.
Secondly, since $B \geq 1$, $\bar{D}$ in Corollary \ref{thm:convdda} is such that $\bar{D} \leq \frac{D_f}{D} + DL^2L'$. 
Using these two facts, the computation for $S$ then follows the steps in Lemma \ref{thm:samp} with $d_hd_v$ replaced by $\tau d_hd_v$. 
Hence we have, 
\begin{equation}\label{eq:sampddacomp}
S(r,\epsilon) \geq \frac{r(\tau d_hd_v)^{3/2}}{t\epsilon^2} 
\end{equation}
To compute the bound for $M$ we follow the same steps in the proof of Lemma \ref{thm:samp}, and end up with the following inequality
\begin{equation}\label{eq:epsdelcompdda} 
Pr\left( \| \nabla_{\bf W} f({\bf W}^{R_c}) \|^2 \geq \epsilon  \bar{D} \right) \leq \frac{(\tau d_hd_v)^{3/4}}{\epsilon\sqrt{N}} 
\end{equation}
Since $N$ is the number of calls for each of the $B$ RSGs, we have $N = \frac{r(\tau d_hd_v)^{3/2}}{\epsilon^2}$ using \eqref{eq:sampddacomp}.
Then we have,
\begin{equation}\label{eq:epsdelcompdda2}
Pr\left( \min_{{1,\ldots,M}} \| \nabla_{\bf W} f({\bf W}^{R_c}) \|^2 \geq \epsilon  \bar{D} \right) \leq \prod_{c=1}^{M} \frac{1}{\sqrt{r}} = \frac{1}{r^{C/2}}
\end{equation}
and hence $M(r,\delta) \geq \frac{\log(\frac{1}{\delta})}{\log(\sqrt{r})} = \ceil[\bigg]{\frac{\log(\frac{1}{\delta})}{\log(\sqrt{r})}}$.
\end{proof}

These results show that, whenever $B$ is chosen as in Lemma \ref{thm:dda}, the convergence rate and sample sizes will improve by $\tau^{3/4}$ and $\tau^{3/2}$ respectively, if the stepsize is appropriate.
The improvements may be much larger whenever $\zeta$ is not unreasonably small (or $\tau$ is not too close to $1$). 

\section{Experiments}
\label{sec:exp}

To evaluate the bounds presented above, we pre-trained a one--layer DA on two computer vision and one neuroimaging datasets -- MNIST digits, Magnetic Resonance Images from Alzheimer's Disease Neuroimaging Initiative (ADNI) and ImageNet. 
These will be referred to as {\it mnist, neuro} and {\it imagenet}.
See supplement for complete details about these datasets, including the number of instances, features and other attributes.
Briefly, {\it neuro} dataset has stronger correlations across its dimensions compared to others, and {\it imagenet} includes natural images and is very diverse/versatile.

Our experiments are two-fold. We first evaluate the non-distributed setting (Corollary \ref{thm:conv}, (\ref{eq:conv})) by computing the expected gradients vs. the number of $\mathcal{SFO}$ calls ($N$) and the network structure ($d_v,d_h$). 
We then evaluate the distributed setup (Corollary \ref{thm:convdda}, (\ref{eq:convdda})) by varying the number of disjoint sub-DAs ($B$) that constitute the network. 
The expectations in \eqref{eq:conv} and \eqref{eq:convdda} are approximated by the empirical average of gradient norm (last $100$ iterations). 
Since we are interested in the trends of convergence rates, all plots are normalized/scaled by the corresponding maximum value of expected gradients. 
Figure \ref{fig:exp} shows these results: the first and second columns correspond to the non-distributed setting and the last column corresponds to the distributed setting. 
Each row represents one of the three datasets considered. 
\begin{figure*}[!ht]\centering
\subfloat[{\footnotesize {\it mnist}; \hspace{0.25mm} Expected gradients vs. $N$}]{\includegraphics[width=58mm]{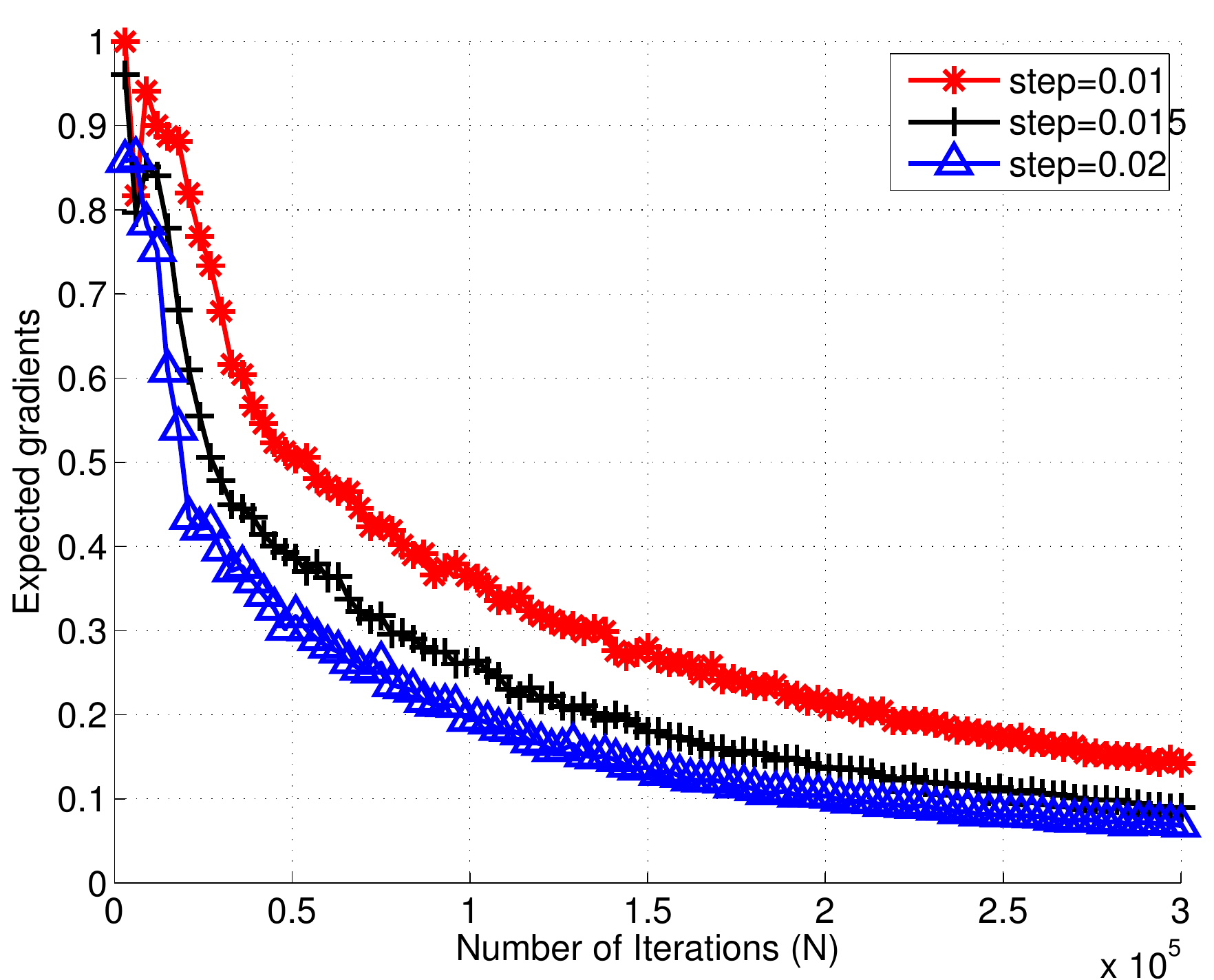}} 
\subfloat[{\footnotesize {\it mnist}; \hspace{0.25mm} Expected gradients vs. $d_v,d_h$}]{\includegraphics[width=58mm]{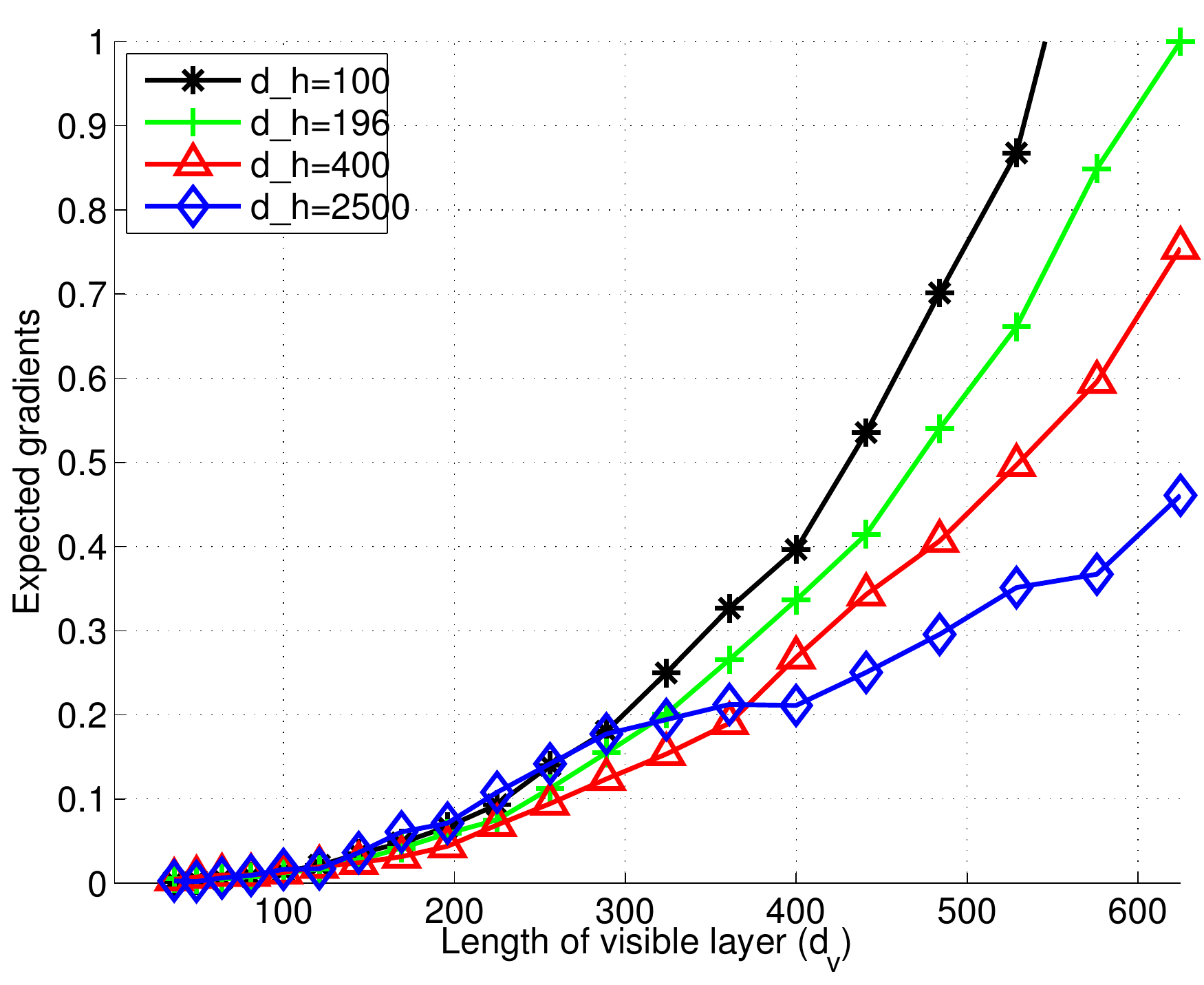}}
\subfloat[{\footnotesize {\it mnist}; \hspace{0.25mm} Expected gradients vs. $B$}]{\includegraphics[width=58mm]{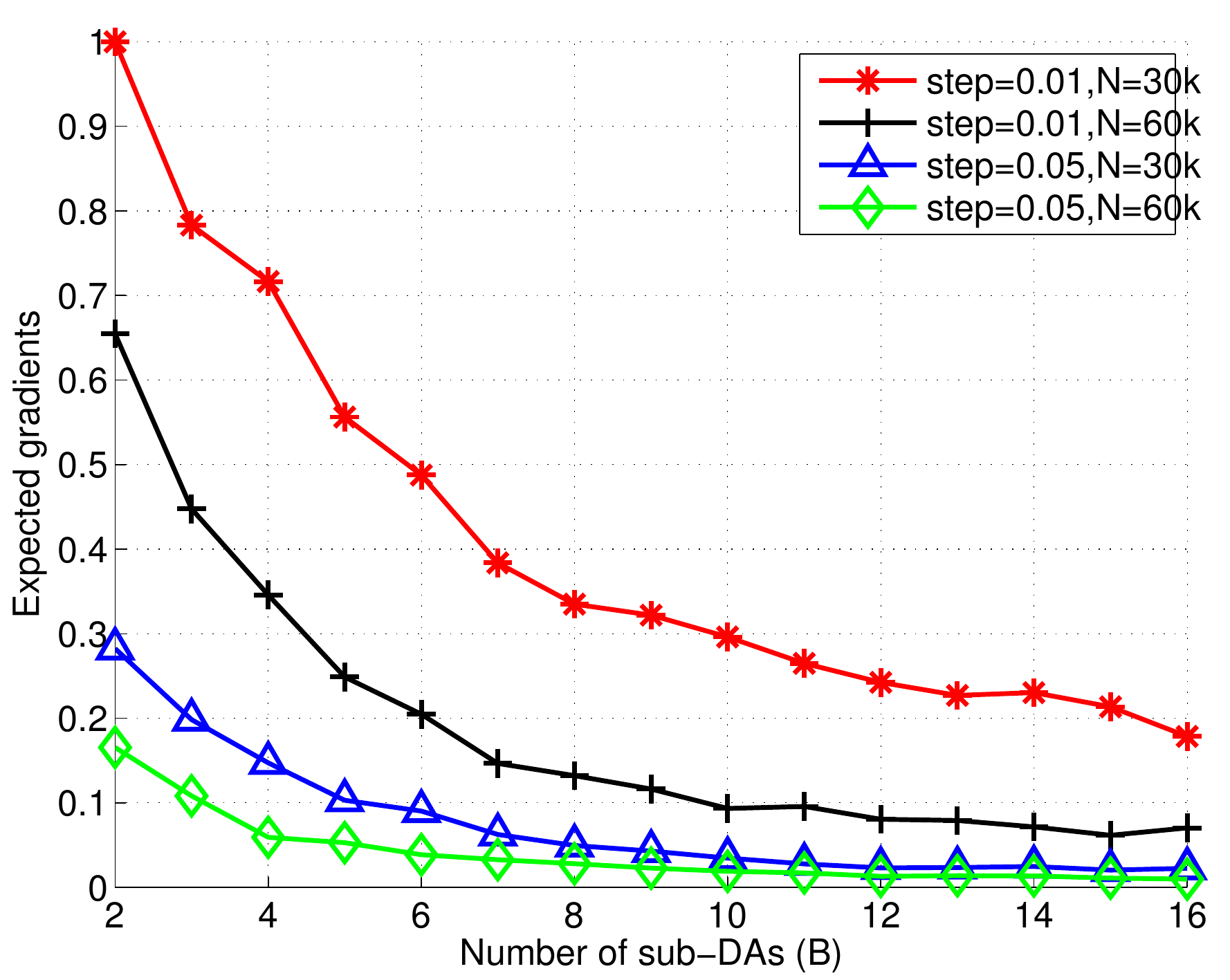}} \\ \vspace{-0.5mm}
\subfloat[{\footnotesize {\it neuro}; \hspace{0.25mm} Expected gradients vs. $N$}]{\includegraphics[width=58mm]{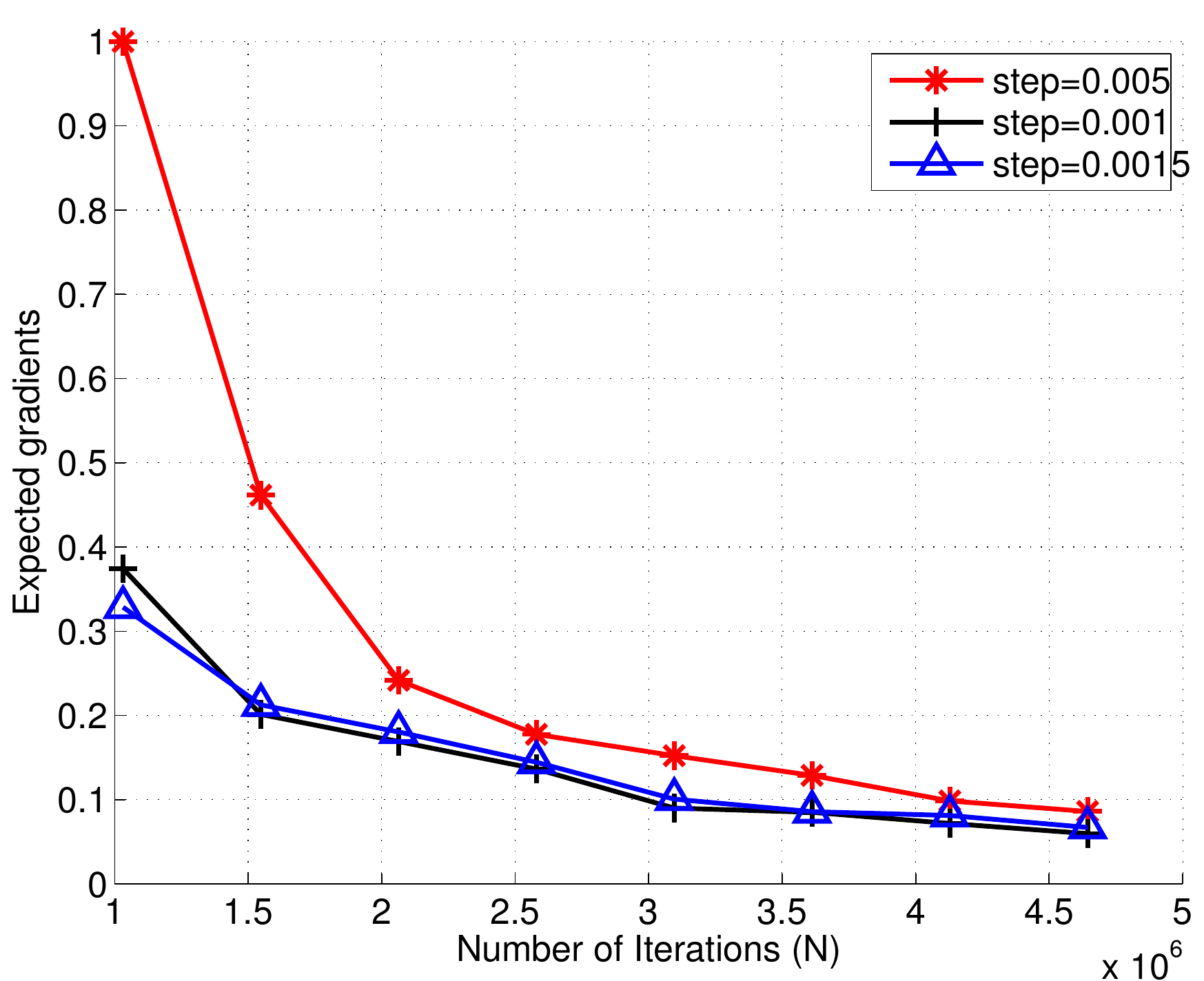}} 
\subfloat[{\footnotesize {\it neuro}; \hspace{0.25mm} Expected gradients vs. $d_v,d_h$}]{\includegraphics[width=58mm]{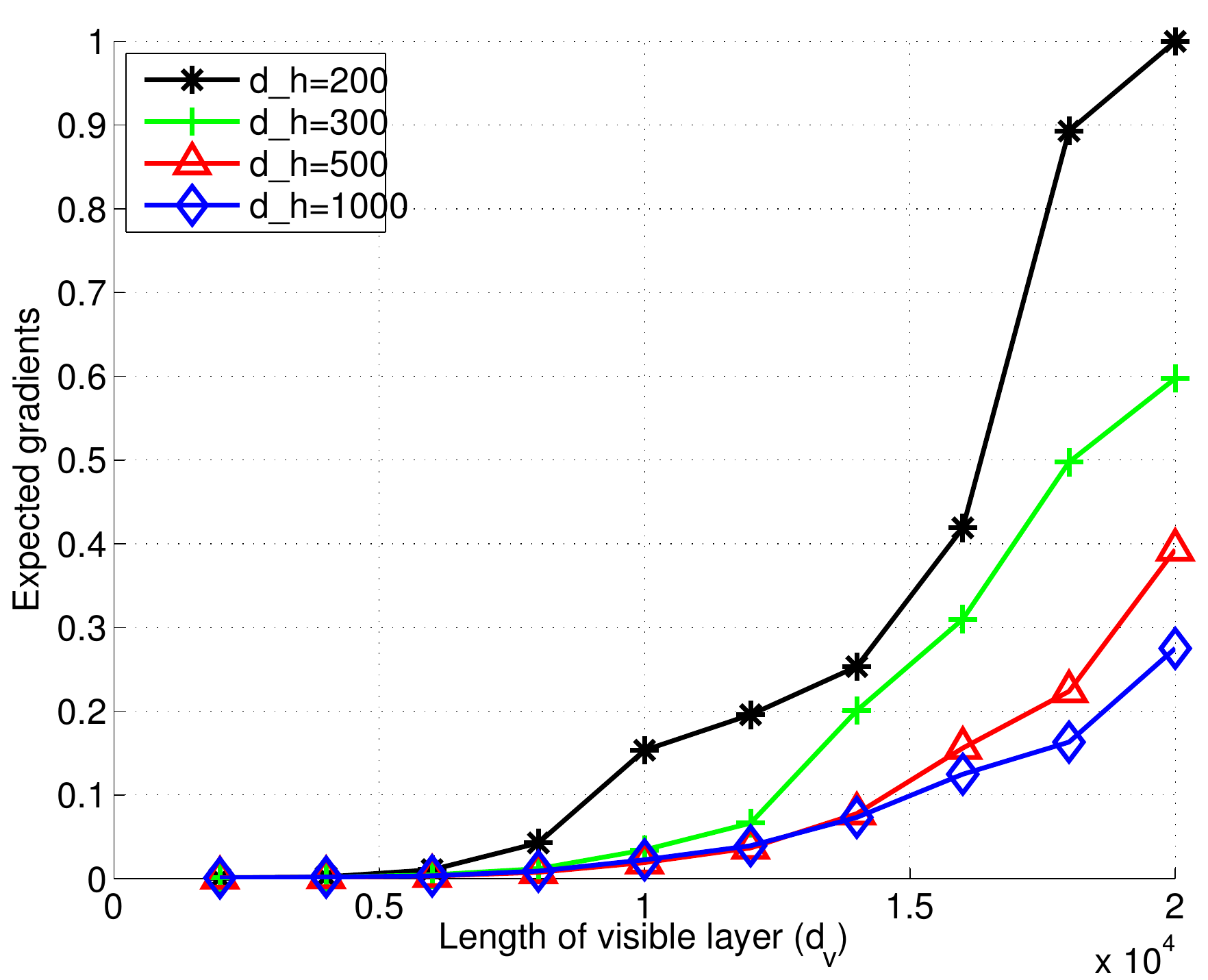}}
\subfloat[{\footnotesize {\it neuro}; \hspace{0.25mm} Expected gradients vs. $B$}]{\includegraphics[width=58mm]{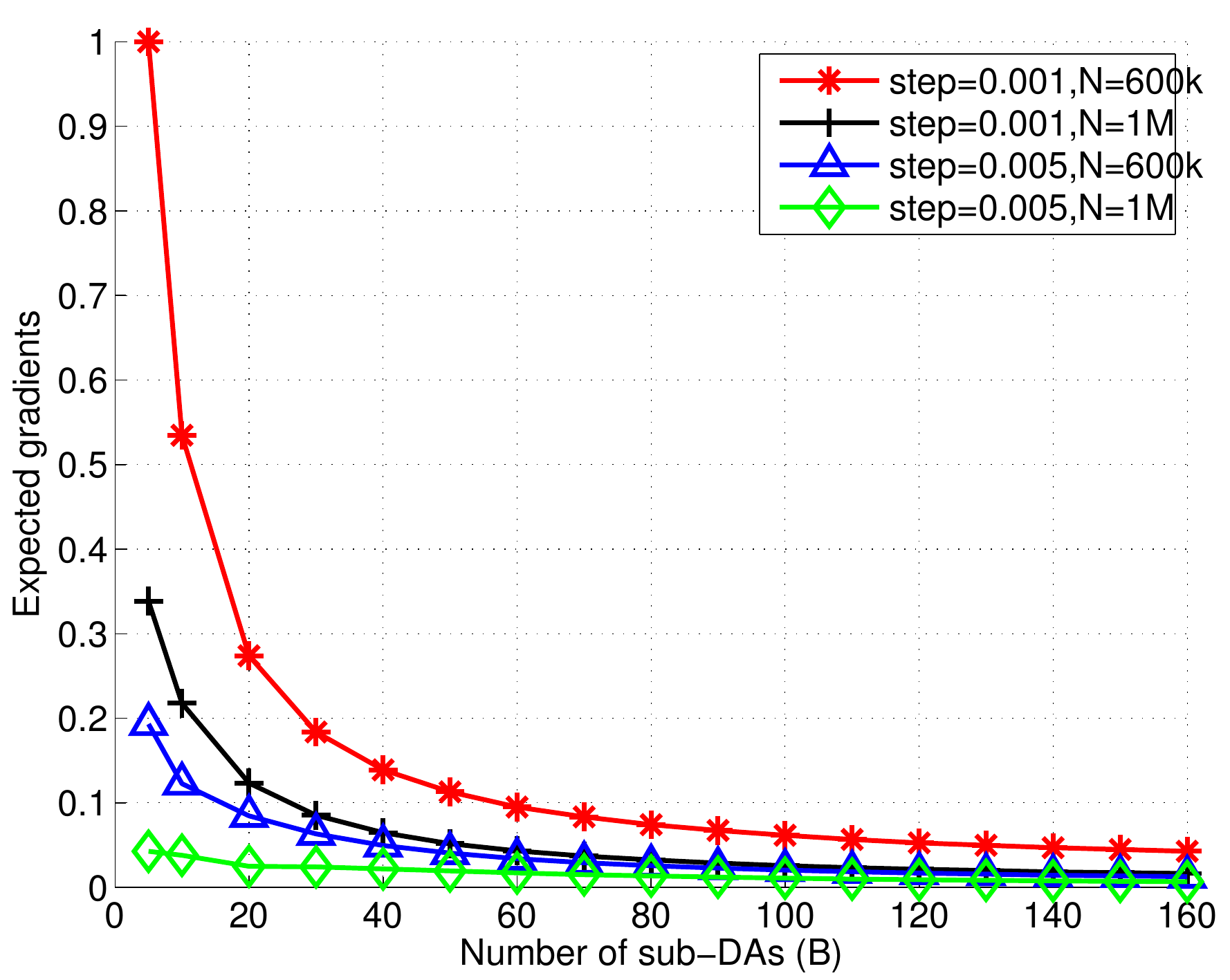}} \\ \vspace{-0.5mm}
\subfloat[{\footnotesize {\it imagenet}; \hspace{0.25mm} Expected gradients vs. $N$}]{\includegraphics[width=58mm]{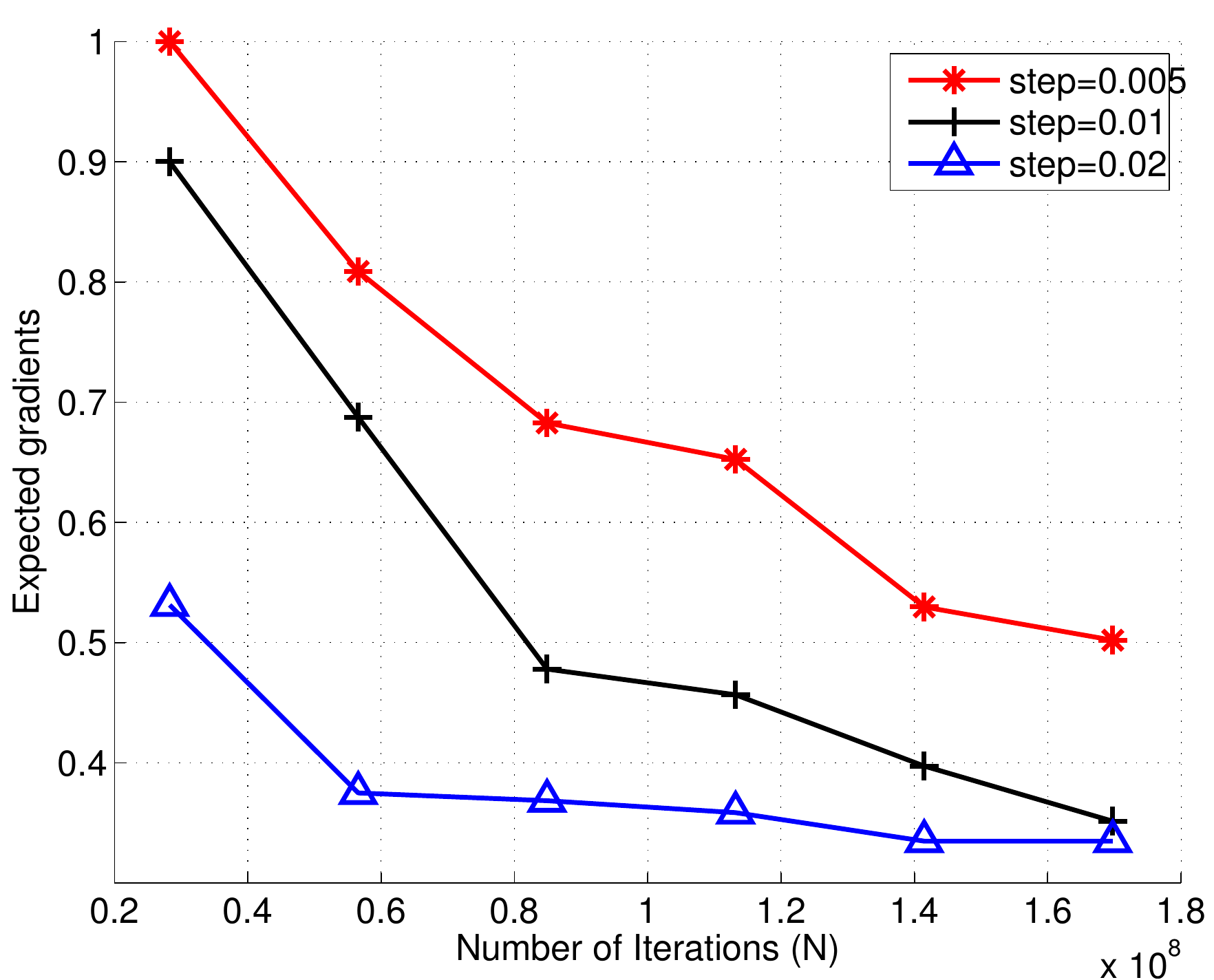}} 
\subfloat[{\footnotesize {\it imagenet}; \hspace{0.25mm} Expected gradients vs. $d_v,d_h$}]{\includegraphics[width=58mm]{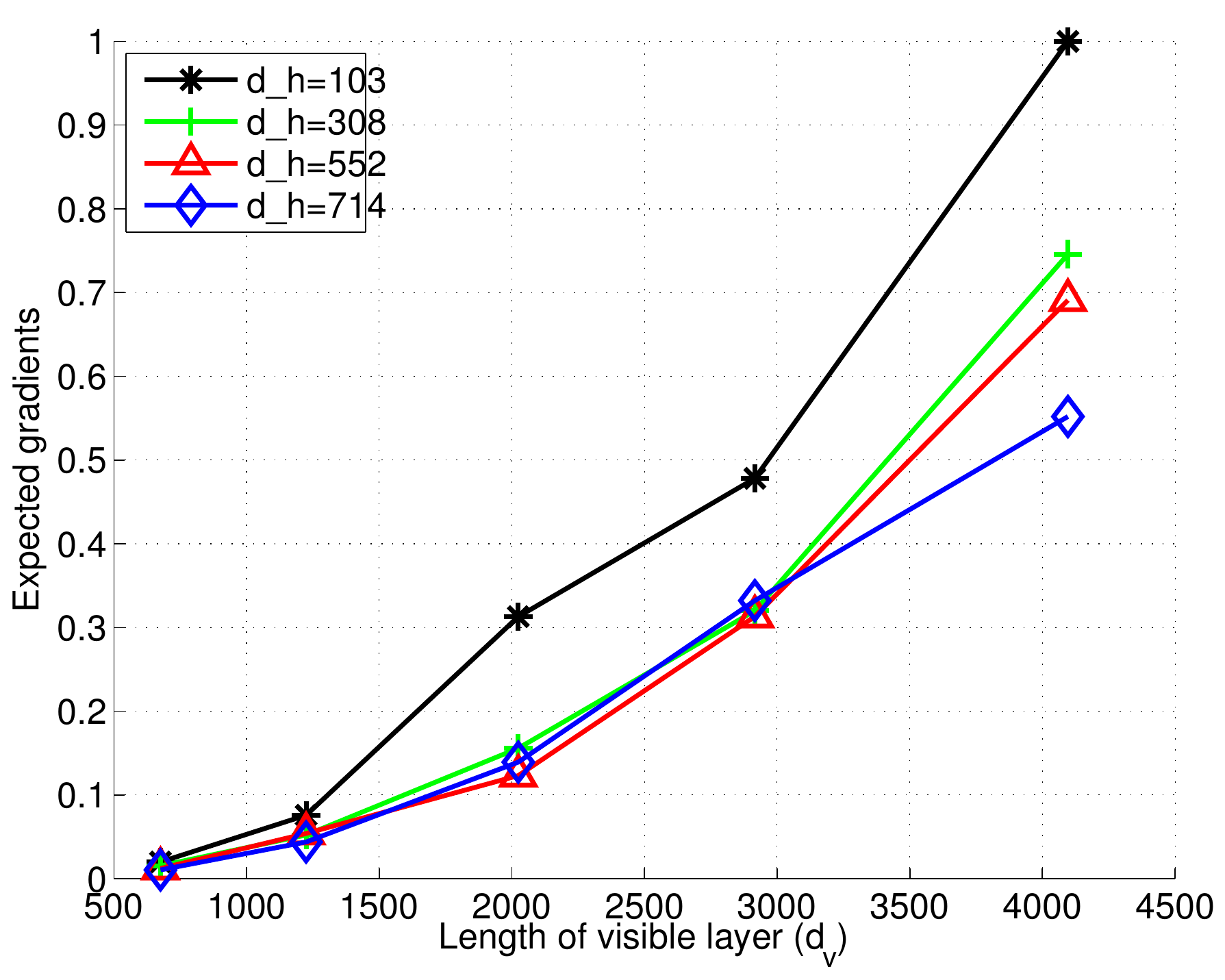}}
\subfloat[{\footnotesize {\it imagenet}; \hspace{0.25mm} Expected gradients vs. $B$}]{\includegraphics[width=58mm]{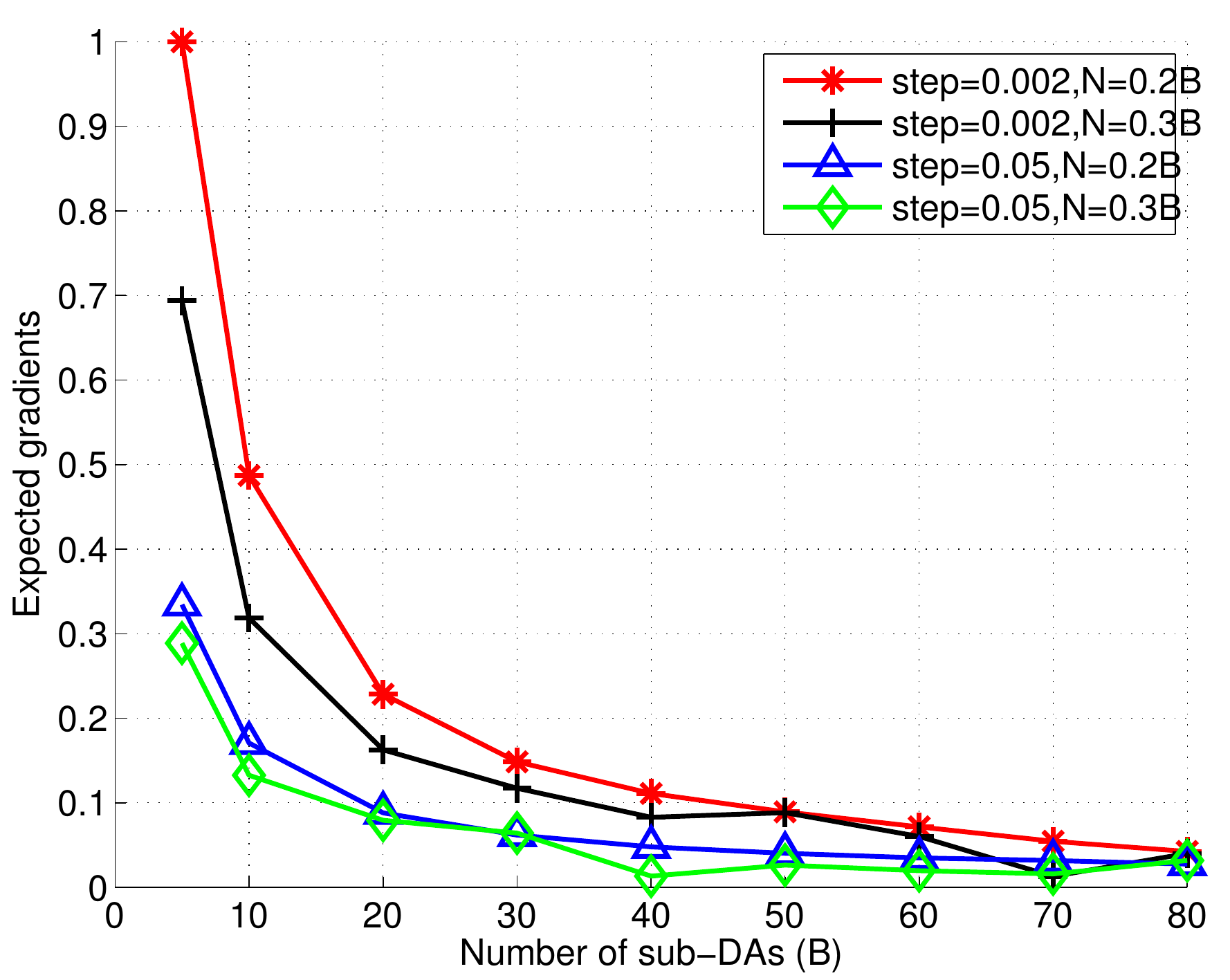}}
\caption{\footnotesize \label{fig:exp} Expected gradients. First (a,d,g) column shows the expected gradients vs the number of $\mathcal{SFO}$ calls $N$, for multiple stepsizes $\gamma$ (corresponding to red, black and blue colors). Second (b,e,h) column shows the expected gradients vs. the size of visible layer $d_v$ for multiple $d_h$s (corresponding to red, blue and black colors). Third column (c,f,i) presents expected gradients vs. the number of sub-DAs ($B>1$) used in a distributed asynchronous setting (for a fixed iterations $N$ and network size $d_h.d_v$). For the results in first and last columns, $d_v$ equals the inherent input data dimensionality (see supplement), and $d_h$ is one-tenth of $d_v$. Top row corresponds to {\it mnist}, second to {\it neuro} and third to {\it imagenet}. All the expected gradients are normalized with the maximum value in the respective plot.}
\end{figure*}

\paragraph{\it Expected gradients vs $N$.}
Figure \ref{fig:exp}(a,d,g) show that the expected gradients decrease as the number of $\mathcal{SFO}$ calls ($N$) increases.
The three curves (red, black, blue) in each plot correspond to different stepsizes.
The expected gradients decrease monotonically for all the curves in Figure \ref{fig:exp}(a,d,g), and their hyperbolic trend as $N$ increases supports the $\mathcal{O}(\frac{1}{\sqrt{N}})$ decay rate presented in (\ref{eq:conv}).
Unlike {\it mnist} and {\it imagenet}, {\it neuro} has stronger correlations across its features, and so shows a decay rate seems to be stronger than $\frac{1}{\sqrt{N}}$ (the red curve in Figure \ref{fig:exp}(d)).
The gradients, in general, also seems to be smaller for larger stepsizes (blue and black curves), which is expected because the local minima are attained faster with reasonably large stepsizes, until the minima are overshot. Supplement shows a plot indicative of this well-known behavior.

\paragraph{\it Expected gradients vs $d_v,d_h$.}
The second column (Figure \ref{fig:exp}(b,e,h)) shows the influence of increasing the length of the visible layer ($d_v$) for multiple $d_h$'s and fixed $N$. 
As suggested by (\ref{eq:conv}), the expected gradients increase as $d_v$ increases. 
This rate of increase (vs. increasing $d_v$ on $x$-axis) seems to be stronger for smaller values of $d_h$ (black and green curves vs. red and blue curves).
Recall that $d_h$ should be ``sufficiently'' large to encode the underlying input data dependencies \cite{paugam1997size, lawrence1998size, bianchini2014nnls}.
Hence the network may under-fit for small $d_h$, and not recover inputs with small error.  
This behavior is seen in Figure \ref{fig:exp}(b,e,h) where initially the expected gradients (across all $d_v$s) gradually decrease as $d_h$ increases (black, green and red curves).
Once $d_h$ is reasonably large, increasing it further tends to increase the expected gradients (as shown by the blue curve which overlaps the others).
$d_h$ may hence be chosen empirically (e.g. using cross validation), so that the network still generalizes to test instances but is not massive (avoids unnecessary computational burden).

\paragraph{\it Does distributed learning help ?}
The last column in Figure \ref{fig:exp} shows the expected gradients in a distributed setting where $x$-axis represents the number of sub-DAs ($B$) into which the whole network is divided.
The number of $B$'s is chosen such that $d_h$ is no larger than twice the size of $d_v$.
Corollary \ref{thm:convdda} presents the bounds with respect to $\tau$ which is the fraction of visible layers used in each of the sub-DAs. 
The results in Figure \ref{fig:exp}(c,f,i) are shown relative to the number of disjoint sub-DAs $B$, which is chosen to be at least $1/\tau$ and follows the conditions in Lemma \ref{thm:dda}. 
Observe that, the expected gradients decay as $B$ increases for all the three datasets considered. 
For a sufficiently large $B$, the decay rate settles down with no further improvement, see Figure \ref{fig:exp}(f,i).
The bounds derived in Section \ref{sec:dist} are based on a synchronous setup. 
In our experiments a central master holds the current updates of the parameters, and the $B$ different sub-DAs pre-train independently on as many as $200$ cores, communicating with the master via message passing. 
The sub-DAs are initialized by running the whole network (in a non-distributed way) for a few hundred iterations.

Figure \ref{fig:time} shows the time speed-up achieved by distributing the pre-training (relative to the non-distributed setting) on {\it neuro} and {\it imagenet}.
Note that the number of sub-DAs used is equal to the number of cores used, which means one sub-DA is pre-trained per core. 
As the number of cores used increases, the speed-up relative to the non-distributed setting increases rapidly up to a certain limit, and then gradually falls back.
This is because for large values of $B$ the communication time between machines dominates the actual computation time.
The speed-up is much higher for datasets with large number of parameters ($>50$mil, red and black curves vs.$15$mil, blue curve). 
Note that the distributed setting gives faster convergence and time speed-up, but does not lose out on generalization error (refer to the supplement for a plot confirming this behavior).  
Lastly, these computational (Figure \ref{fig:exp}(c,f,i)) and time speed-up (Figure \ref{fig:time}) improvements of distributed setup are in agreement with existing observations \cite{raina2009large, dean2012large}. 
Overall, the results in Figures \ref{fig:exp} and \ref{fig:time}, in tandem with existing observations \cite{bengio2009learning, erhan2009difficulty, erhan2010does, vincent2010stacked, dean2012large} provide strong empirical support to the convergence and sample size bounds constructed in Sections \ref{sec:uppersamp} and \ref{sec:dist}.

\begin{figure}[!t]\centering
\includegraphics[width=63mm]{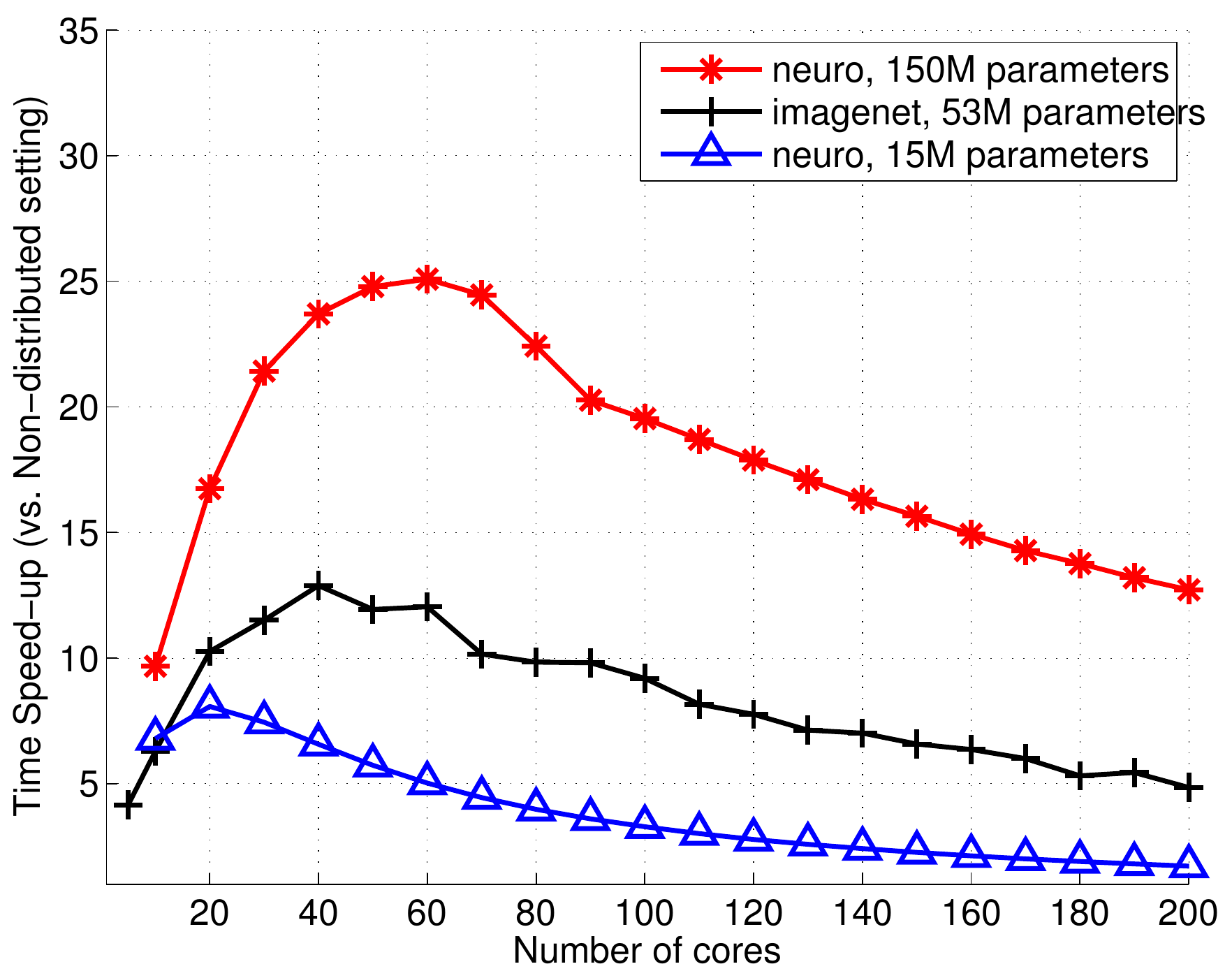}
\caption{\footnotesize \label{fig:time} The relative time speed-up achieved by distributed pre-training (vs. non-distributed) as a function of the number of cores ($x$-axis), which in this experiment is equal to the number $B$ of sub-DAs (see Lemma \ref{thm:dda}) i.e. each of core works on one sub-DA. Curves correspond to different number of parameters. Step-sizes are scaled according to \ref{eq:gammabounds}, while $N$ is fixed for each curve.} 
\end{figure}

\section{Conclusion}
\label{sec:con}

We analyzed the convergence rate and sample size estimates of gradient based learning of deep architectures. 
The only assumption we make is on the Lipschitz continuity of the loss function.
We provided bounds for classical and distributed pre-training for Denoising Autoencoders, and the experiments support the suggested behavior. 
We believe that our results complement a sizable body of work showing the success of empirical pre-training in deep architectures 
and identifies a number of interesting directions for additional improvements -- both on the theoretical side as well as the design of practical large scale pre-training.

\section*{Appendix (Supplementary Material)}

\paragraph{\bf Datasets Description}

The three datasets that were used, {\it mnist}, {\it neuro}, {\it imagenet}, correspond to the small$d$-large$n$, large$d$-small$n$ and large$d$-large$n$ setups respectively ($d$ is the number of data dimensions, $n$ is the number of data instances). 
\begin{itemize}
\item {\it mnist}: This famous digit recognition dataset contains binary images of hand-written digits ($0-9$). We used $10^4$ of these images which are part of the mnist training data set (http://yann.lecun.com/exdb/mnist/). The training data contains approximately equal number of instances for each of the ten classes. Each image is $784$ pixels/dimensions, and the signal in each pixel is binary. No extra preprocessing was done to the data. 
\item {\it neuro}: This neuroimaging dataset is a prototypical example of dataset with very large number of features, but small number of instances. It comprises of Magnetic Resonance Imaging (MRI) data from Alzheimer's Disease Neuroimaging Initiative study from a total of $534$ subjects. Each image is three-dimensional of size $256 \times 256 \times 176$. Each voxel in this 3D space corresponds to water-level intensity in the brain, and the signal is positive scalar. Standard pre-processing is applied on all the images, which involves stripping out grey matter and normalizing to a template space (called MNI space). Refer to Statistical Parametric Mapping Tool (SPM8, http://www.fil.ion.ucl.ac.uk/spm/doc/) for this standardised procedure. The resulting processed images are sorted out according to the signal variance. For the experiments in thie work, we picked out the top (most variant) $25\%$ of the features/voxels, which amounted to $3 \times 10^4$ features. Even within this setting the number of features is much larger than the number of instances available ($534$).  
\item {\it imagenet}: This well-known dataset comprises of natural images from various types of categories collected as apart of WordNet hierarchy. It comprises of more than $14$ Million images, broadly categorized under more than $20$ thousand synsets (http://www.image-net.org/). We used imaging data from five of the largest categories contained in the imagenet database. This amount to $> 7000$ synsets/sub-categories and approximately $5$ million images. As a pre-procesing step, we resized all images to $128 \times 128$ pixels, and centered each of the $16384$ dimensions.  
\end{itemize}

\begin{figure}[!h]\centering
\subfloat[]{\includegraphics[width=60mm]{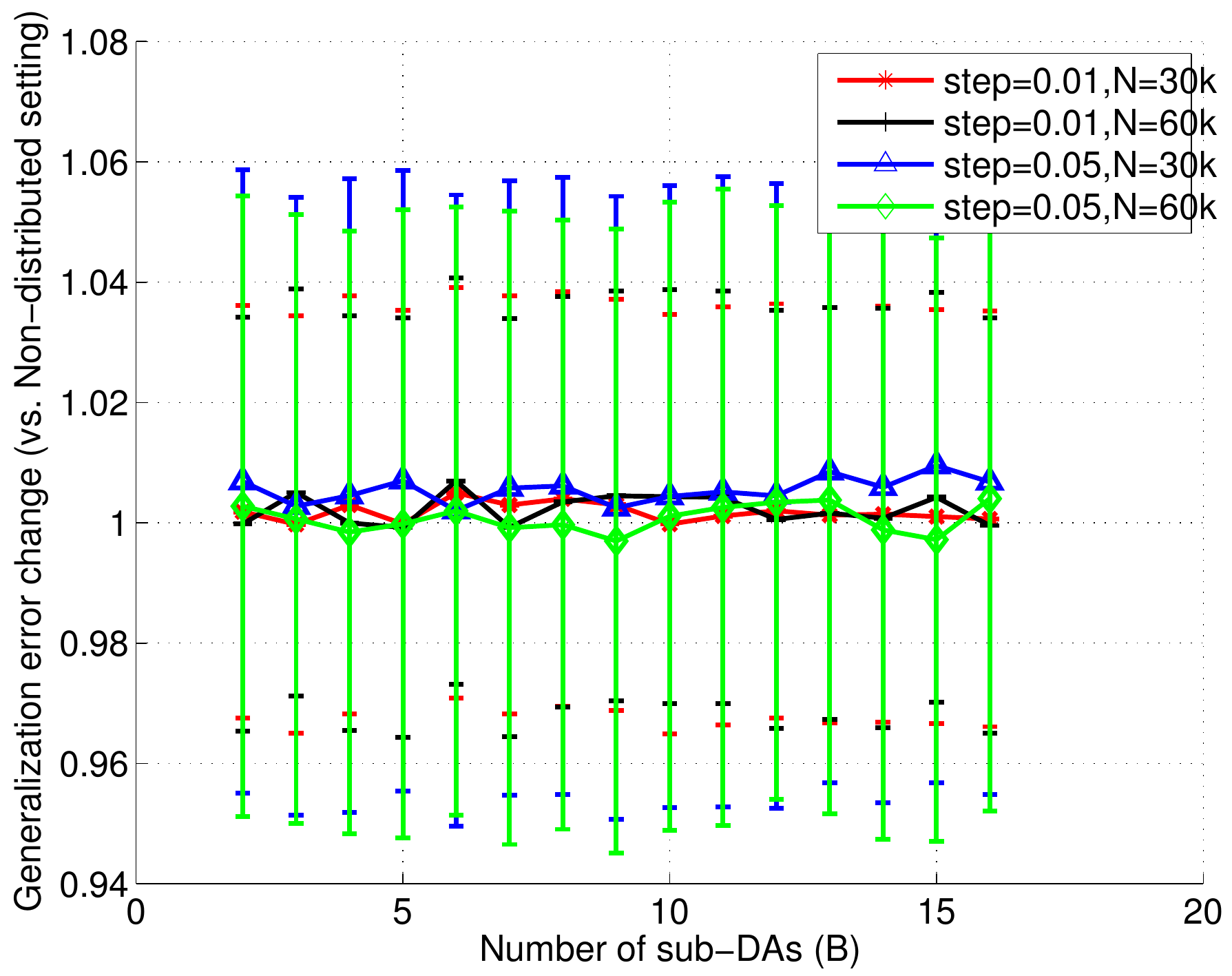}} \quad \quad
\subfloat[]{\includegraphics[width=60mm]{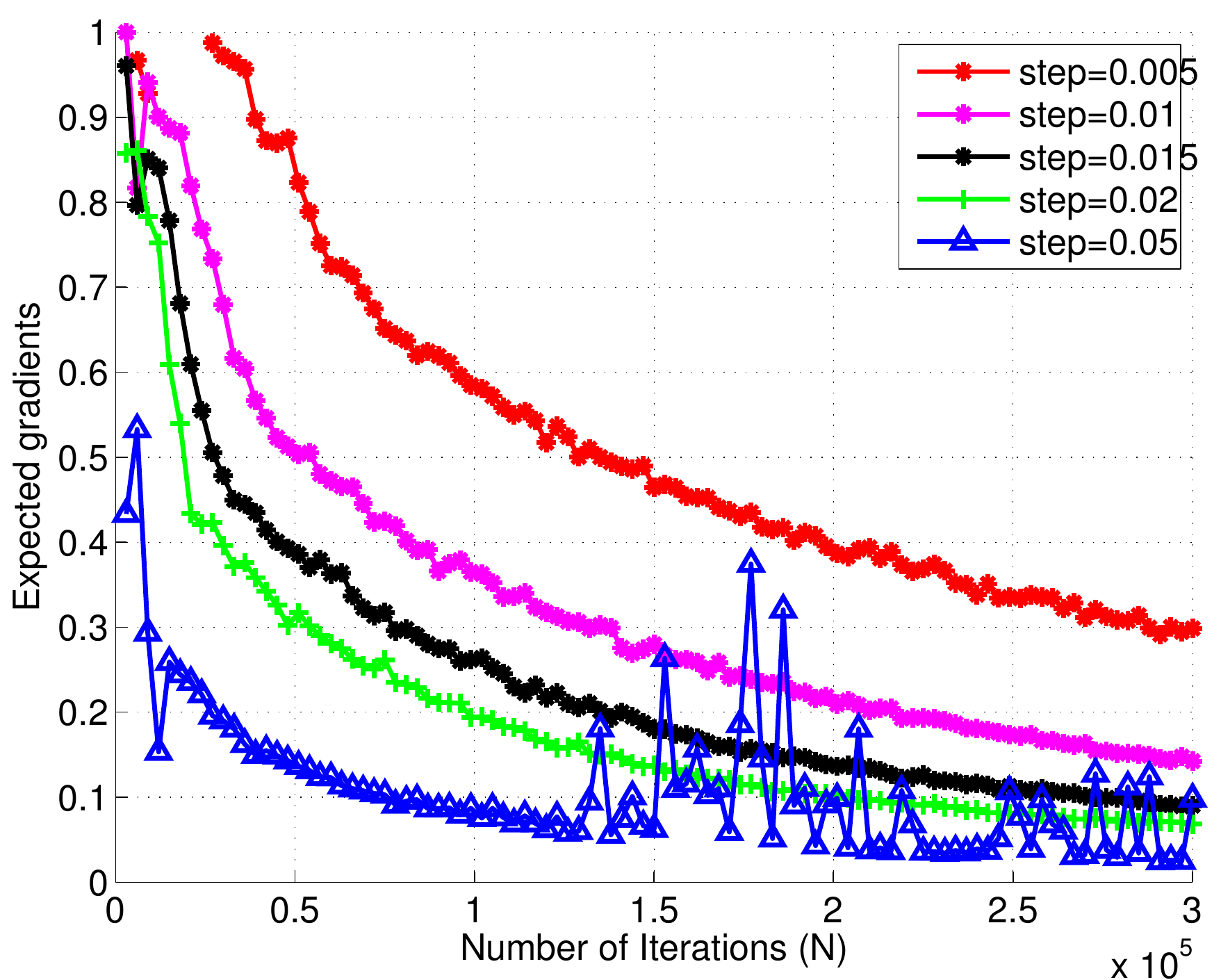}} 
\caption{\footnotesize \label{fig:errors} (a) Distributed setup does not lose on generalization error. The four curves correspond to the ratio of test-set reconstruction errors for distributed pre-training ($B>2$) to the non-distributed case. The error-bars correspond to $10$ fold cv errors computed using $10$ different test-sets.
(b) Expected gradients vs the number of $\mathcal{SFO}$ calls $N$, for multiple stepsizes $\gamma$ (corresponding to the four different colors). The trends show that as stepsize increases the expected gradients decrease, and beyond a resonably large stepsize (gree curve) the gradients overshoot local optima (blue curve).}
\end{figure}

\bibliographystyle{natbib}
\bibliography{pretraining}

\end{document}